\documentclass[conference]{IEEEtran}
\IEEEoverridecommandlockouts
\usepackage{cite}
\usepackage{amsmath,amssymb,amsfonts}
\usepackage{graphicx}
\usepackage{textcomp}
\def\BibTeX{{\rm B\kern-.05em{\sc i\kern-.025em b}\kern-.08em
		T\kern-.1667em\lower.7ex\hbox{E}\kern-.125emX}}

\usepackage{comment}

\usepackage{blindtext}

\usepackage{colortbl}
\usepackage{multirow}

\usepackage{algorithm}
\usepackage{algpseudocode}
\usepackage{hyperref}
\usepackage{url}
\usepackage{epstopdf}
\usepackage{physics}
\usepackage{mathtools}
\usepackage{dsfont}
\usepackage{threeparttable}
\usepackage{makecell}
\usepackage{color,soul}
\soulregister\cite7
\soulregister\ref7
\soulregister\pageref7
\usepackage{bm}
\usepackage{paralist}
\usepackage{tablefootnote}
\usepackage[table,xcdraw]{xcolor} 

\usepackage{subcaption}
\usepackage{amsthm}

\renewcommand{\footnotesize}{\scriptsize}

\newtheorem{lemma}{Lemma}

\newtheorem{example}{Example}

\newcommand{\calE}{\mathcal{E}}
\newcommand{\calF}{\mathcal{F}}

\newcommand{\calX}{\mathcal{X}}
\newcommand{\calY}{\mathcal{Y}}

\newcommand{\bg}[1]{\textcolor{teal}{BG: {#1}}} 
\newcommand{\ag}[1]{\textcolor{green}{AG: {#1}}} 

\usepackage{soul}

\newcommand{\ignore}[1]{}

\DeclareMathOperator*{\argmax}{arg\,max}

\newtheorem{definition}{Definition}
\newtheorem{problem}{Problem}

\newcommand{\alg}{{\sf MISFEAT}}
\newcommand{\algspace}{{\sf MISFEAT} }
\newcommand{\algsample}{{\sf RANDWALK }}
\newcommand{\algsamplenospace}{{\sf RANDWALK}}
\newif\ifTechReport
\TechReporttrue
\newif\ifNotTechReport

\begin{document}
%

\title{\alg: Feature Selection for Subgroups with Systematic Missing Data}
\author{
\IEEEauthorblockN{Bar Genossar$^{+}$, Thinh On$^{*}$, Md. Mouinul Islam$^{*}$, Ben Eliav$^{+}$, Senjuti Basu Roy$^{*}$, Avigdor Gal$^{+}$}
\IEEEauthorblockA{\textit{Technion -- Israel Institute of Technology$^{+}$}, \textit{New Jersey Institute of Technology}$^{*}$ \\
sbargen@campus.technion.ac.il,to58@njit.edu,mi257@njit.edu, \\
ben.eliav@campus.technion.ac.il,senjutib@njit.edu,avigal@technion.ac.il}
}

\maketitle

\begin{abstract}
We investigate the problem of selecting features for datasets that can be naturally partitioned into {\em subgroups} ({\em e.g.}, according to socio-demographic groups and age), each with its own dominant set of features. Within this subgroup-oriented framework, we address the challenge of \emph{systematic missing data}, a scenario in which some feature values are missing for all tuples of a subgroup, due to flawed data integration, regulatory constraints, or privacy concerns. Feature selection is governed by finding \emph{mutual Information}, a popular quantification of correlation, between features and a target variable. Our goal is to identify top-$K$ feature subsets of some fixed size with the highest joint mutual information with a target variable. In the presence of systematic missing data, the closed form of mutual information could not simply be applied. We argue that in such a setting, leveraging relationships between available feature 
mutual information within a subgroup or across subgroups can assist inferring missing mutual information values. 
We propose a generalizable model based on \emph{heterogeneous graph neural network} to identify interdependencies between feature-subgroup-target variable connections by modeling it as a multiplex graph, and employing information propagation between its nodes. We address two distinct scalability challenges related to training and propose principled solutions to tackle them. Through an extensive empirical evaluation, we demonstrate the efficacy of the proposed solutions both qualitatively and running time wise.
\end{abstract}

\section{Introduction and Motivation}
Features, measurable properties of a phenomenon, are useful in data science for data exploration and model building. Feature selection, the process of choosing informative and discriminating features from a large set, is a key step in data science pipelines, focusing on retaining features that provide the greatest benefit for learning~\cite{li2017feature, chandrashekar2014survey, dash1997feature}.
Mutual information (MI) is a model-agnostic measure that quantifies the dependency between two variables and has been widely used in feature selection~\cite{chen2021efficient, vergara2014review, pascoal2017theoretical, huang2007hybrid, battiti1994using, fleuret2004fast}. Intuitively, features with higher MI values relative to the target variable are often more important.

In this work, we focus on feature selection using MI that is challenged by the presence of (1) distinct subgroups and (2) systematic missing data. A \emph{subgroup} 
constitutes a group of records demonstrating identical characteristics pertinent to an application at hand. Subgroup analysis is common in user group analytics~\cite{DBLP:journals/tkde/Omidvar-Tehrani20} and we illustrate the notion of subgroups next using medical cohort analytics~\cite{MUNSHI2017187,sun2014use}. 


\begin{example}
\label{ex:Readmission_subgroups} 
In medical cohort analysis, given a subgroup (cohort), experts seek answers to three
typical questions~\cite{DBLP:journals/vldb/Omidvar-Tehrani20}, namely predicting some future status of patients in a cohort, interpreting a phenomenon using a cohort, and seeking similar cohorts. To effectively pursue these tasks, 
models are empowered by the most informative (for within subgroup analysis) and discriminating (for between subgroup analysis) features. 
\begin{table*}[t]
	\centering
	\caption{Example Cardiovascular Health Dataset with Systematic Missing Data}
	\begin{tabular}{|c|c|c|c|c|c|c|c|c|c|}
		\hline
		\textbf{Patient ID} & \textbf{Age} & \textbf{Ethnicity} & \textbf{Blood Pressure} & \textbf{Family History} & \textbf{Body Weight} & \textbf{Smoking} & \textbf{Cholesterol} & \textbf{Readmission} \\
		\hline
		\rowcolor{blue!20} 1 & 45 & Asian & Normal & Yes & Overweight  & Non-smoker & Normal & No \\
		\rowcolor{yellow!30} 2 & 62 & Caucasian & Null & No & Overweight  & Smoker & Normal & Yes \\
		\rowcolor{red!30} 3 & 35 & Caucasian & Normal & Yes & Normal  & Non-smoker & Null & Yes \\
		\rowcolor{yellow!30} 4 & 50 & Caucasian & Null & Yes & Obese  & Non-smoker & High & Yes \\
		\rowcolor{green!30} 5 & 30 & Asian & Null & No & Normal & Smoker & Null & No \\
		\rowcolor{red!30} 6 & 28 & Caucasian & Null & Yes & Overweight & Non-smoker & Null & Yes \\
		\rowcolor{blue!20} 7 & 55 & Asian & Hypertension & Yes & Overweight & Smoker & Low & No \\
		\rowcolor{green!30} 8 & 28 & Asian & Null & No & Normal & Smoker & Null & No \\
		\rowcolor{yellow!30} 9 & 60 & Caucasian & Null & Yes & Obese & Non-smoker & Normal & Yes \\
		\rowcolor{green!30} 10 & 38 & Asian & Null & Yes & Overweight & Smoker & Null & Yes \\
		\rowcolor{blue!20} 11 & 54 & Asian & Prehypertension & Yes & Normal & Non-smoker & Null & Yes \\
		\rowcolor{red!30} 12 & 30 & Caucasian & Prehypertension & No & Underweight & Smoker & Null & No \\
		\hline
	\end{tabular}
	\label{tab:ReadmissionTable}
\end{table*}

Table~\ref{tab:ReadmissionTable} represents an instance of patient health records. Consider a task of predicting readmission (target variable) to a hospital within 30 days after initial patient discharge. 
Cohort analysis over this data excerpt may be based on a combination of age range and ethnic group. A possible such characterization (marked with different colors in the table) entails four subgroups: Asian aged 40 and below, Caucasian aged 40 and below, Asian above 40, and Caucasian above 40. In the context of feature selection, the goal is to identify for each of the subgroups predictors that exhibit high informativeness of being readmitted.
Focusing on the {\em Family History} feature, whenever its value is positive for the subgroups of 40 and below,  the corresponding {\em Readmission} value is consistently ``Yes,'' implying 
that {\em Family History} is a predictive feature of readmission within these subgroups. However, the same pattern does not hold for other subgroups. 
This nuanced subgroup-specific variation highlights the importance of tailoring feature selection to distinct subgroups.
\end{example}

Missing data is a well-known phenomenon in data science, attributed to data issues such as 
measurement errors, manual data entry issues, data integration flaws and intentional non-responses. Missing data also complicates feature selection, due to the difficulty in assessing feature relevance with incomplete information~\cite{baraldi2010introduction, zhu2021efficient, little2019statistical}. While data can be missing at random, we are especially interested in systematic missingness~\cite{newman2014missing}, referring to a pattern that follows a discernible trend or mechanism. Systematic missingness often results from regulatory constraints. For example, privacy regulations might require the removal of sensitive attribute data ({\em e.g.}, gender). It can also stem from common domain practices, {\em e.g.}, medical guidelines often dictate routinely conducting different diagnostic tests for different age groups. These regulatory constraints and domain-specific practices often result in a complete absence of feature values for certain subgroups, while they remain available for others. In Table~\ref{tab:ReadmissionTable}, systematic missing data is illustrated in the complete absence of cholesterol levels data for sub-groups age 40 and below. Also, whenever data is collected and integrated from multiple jurisdictions, sensitive attributes data may be completely missing in some data sources.

We wish to identify feature subsets (of limited cardinality) {\em that are most informative to a target variable for each sub-group} in the presence of systematic missing data. 
Feature set informativeness is quantified by the joint MI between predicting features and a target variable, following~\cite{salam2019human,mi1,mi2,mi3}. 
We argue that this task could avoid the extra cost of collecting additional data, and without imputing missing values~\cite{impute1, impute2, mi1, qian2015mutual}, which may be ineffective. Instead, we propose to directly estimate feature subsets' MI that are possibly systematically missing in a sub-group, a stark departure from existing works and a fundamental novelty of our work. 
\smallskip \noindent {\bf Challenges.} (1) 
Features often interrelate and inter-depend, yet their effects can manifest differently on the target variable across subgroups. With an exponential number of feature subsets, the 
first challenge involves {\em designing a generalizable model} to exploit feature interdependence across subgroups.\\
\noindent (2)
The MI upward closure property~\cite{salam2019human} states that the joint MI of a feature set (with respect to a target variable) is never greater than that of any of its supersets. Therefore, the designed model should be cognizant of and benefit from the upward closure property when completing missing MI values of feature supersets and subsets. 
\noindent (3) 
Computing MI values over feature subsets requires enumeration over a power set, which may be prohibitively expensive both computing time and storage space wise. So, the challenge is to investigate the opportunity of sharing computation during training.

\smallskip \noindent {\bf Contributions and structure.}
We design a multiplex graph~\cite{hamilton2020graph, genossar2023flexer, yu2022multiplex} (Section~\ref{sec:graphConstruction}) where features in each subgroup are represented by a distinct graph layer, and train a heterogeneous Graph Neural Network (GNN)~\cite{hamilton2020graph,kipf2016semi, zhou2020graph, schlichtkrull2018modeling} over it (Section~\ref{subsection:training}). We cast the MI estimation problem as a graph representation learning task~\cite{he2022webmile}, estimating MI scores for feature combinations. 
The trained model mostly obey the MI upward closure property, 
effectively increasing the model's predictive power and reducing the effort needed for exact MI computation. 
{\em The model is generalizable}, handling both systematic and random missing data, and could be trained on other model agnostic techniques (e.g., feature selection using Pearson correlation~\cite{cohen2009pearson}). This novel model, to the best of our knowledge, has not been devised before. 


Training the GNN is challenged by graph size,  
 with a node cardinality that is the order of a power set of number of features for each of multiple subgroups (Section~\ref{sec:sampling}). 
 We attend to two distinct computational opportunities. First, we exploit {\em the relationship between MI and joint entropy,  and the chain rule of joint entropy} to demonstrate that the MI of any set of features over the power set could be expressed as a linear function of joint entropy and conditional entropy~\cite{cover1991entropy} over a set of constructs. If these constructs are precomputed and materialized, they could be reused repeatedly during training, allowing cost and extensive speed up. Next, we study a sub-problem that determines which feature subsets to compute given a budget of MI calculations per layer. We reason that the selected subset should be the one that are uniform random sample of the distribution of MI of all nodes. We present a lazy random walk based algorithm that is guaranteed to converge to a unique stationary distribution producing uniform random samples over the search space (Section~\ref{sec:sampling}). In Section~\ref{subsection:topk}, we demonstrate a possible use of the trained model to obtain top-$K$ feature sets with maximum MI score.

A thorough empirical analysis, conducted on both synthetic and real-world datasets, corroborates the efficacy of our approach. Specifically, we show: (1) robustness of the proposed model, which remains effective when a large number of features are systematically missing, compared to imputation based, neural network-based, and Markov blanket-based baselines; (2) efficiency of the sampling algorithm, overcoming a major computational bottleneck while satisfying the uniformity requirement; and (3) proposed solution scalability under varying parameters. 
The entire code is publicly available.\footnotemark\footnotetext{\url{https://github.com/BarGenossar/MISFEAT/}}




Section~\ref{sec:Preliminaries} offers necessary background, 
followed by data model and problem definition (Section~\ref{sec:problem_definition}). Related work is covered in Section~\ref{sec:related_work} and Section~\ref{sec:conclusion} concludes the paper.



\section{Preliminaries}
\label{sec:Preliminaries}
In this section we introduce the necessary background on feature selection and MI (Section~\ref{subsection:FeatureSelection}), 
and GNNs (Section~\ref{sec:GNNs}). Table~\ref{tab:symbol} provides a summary of notations. 

\subsection{Feature Selection and Mutual Information}\label{subsection:FeatureSelection}
A feature selection task involves a dataset $\mathbb{D} = \left\lbrace (x_i, y_i) \mid 1\leq i \leq n\right\rbrace $, where $x_i$ is the feature vector of the $i$-th tuple, $y_i$ is its corresponding label (which may be discrete or continuous), and $n$ is the number of tuples in the dataset.
Let 
$F$ denote the initial feature set. 
The objective of feature selection is to identify 
a subset of features $F^m=\{f_1,f_2,\dots,f_m\}\subset F$ (of some fixed, application-dependent size $m$) 
that maximizes the predictive power of a machine learning model with respect to a predefined metric $M$. This metric can either depend on model performance or be model-agnostic,
	and is computed using $\mathbb{D}^m = \left\lbrace(x_i^m, y_i) | 1\leq i \leq n\right\rbrace$, a vertical subset of $\mathbb{D}$ containing only the features in $F^m$, with $x_i^m$ being the projection of the feature vector to the selected features in $F^m$.

\ignore{The feature selection process benefits from using information theory measures
to assess the importance of feature combinations and selecting the final feature set accordingly. }
In this study, we 
select feature subsets that maximize 
MI~\cite{chen2021efficient, Vergara2014, pascoal2017theoretical}, 
commonly used as an indicator of the dependence between two random variables. MI is a model-agnostic feature importance assessment tool that quantifies the amount of information gain of one (or more) random variables when another random variable is observed. Higher MI values indicate greater feature importance. 
{\em Joint MI} of an $m$-size feature set $F^m$ to a target variable $Y$ is defined as

\begin{footnotesize}
	\begin{equation}\label{eq:MI}
		I(F^m; Y) = \sum_{\forall i \in f_1} \sum_{\forall j \in f_2} \ldots  \sum_{\forall m \in f_m} \sum_{y \in \calY} P(i,j\ldots m,y) \log \frac{P(i,j\ldots m, y)}{P(i,j\ldots m)P(y)}
	\end{equation}
\end{footnotesize}
\begin{example}\label{ex2}
Using Table~\ref{tab:ReadmissionTable} we illustrate MI computation between \emph{Family History} and the target variable \emph{Readmission}. The domain of both features is $\{\text{Yes}, \text{No}\}$. The marginal distribution of \emph{Family History} is $P(\text{Yes}) = \frac{2}{3}$, $P(\text{No}) = \frac{1}{3}$, and for \emph{Readmission} is $P(\text{Yes}) = \frac{7}{12}$, $P(\text{No}) = \frac{5}{12}$. The joint distribution of the variables (ordered as before) 
$P(\text{Yes, Yes}) = \frac{1}{2}$, $P(\text{Yes, No}) = \frac{1}{6}$, $P(\text{No, Yes}) = \frac{1}{12}$, $P(\text{No, No}) = \frac{1}{4}$. 
Plugging values into Eq.~\ref{eq:MI} (with $m=1$) yields:
\begin{align}
	& I(\text{Family History}; \text{Readmission}) = \nonumber\\
	& \frac{1}{2}\log \frac{1/2}{(2/3) \cdot (7/12)} + \frac{1}{6}\log \frac{1/6}{(2/3) \cdot (5/12)} \nonumber\\
	& + \frac{1}{12}\log \frac{1/12}{(1/3) \cdot (7/12)} + \frac{1}{4}\log \frac{1/4}{(1/3) \cdot (5/12)} \simeq 0.23 \nonumber
\end{align}
The joint MI of \emph{\{Family History, Body Weight\}} and \emph{Readmission} is computed using the joint distributions of \emph{\{Family History, Body Weight\}} and \emph{\{Family History, Body Weight, Readmission\}}, as well as \emph{Readmission} marginal distribution. 
\end{example}

\begin{footnotesize}
	\begin{table}[t]
		\begin{tabular}{c|l}
			\hline
			\textbf{Symbol} & \textbf{Explanation} \\
			\hline
			$\mathbb{D}$ & Dataset \\
			$F$ & Feature set \\
			$F^{\prime} \subset F$ & Feature subset for defining subgroups \\
			$P=\{p_1,p_2,\dots,p_{|P|}\}$ & Complete set of minterm predicates over $F^{\prime}$ \\
			$\mathbb{D}_i \subset \mathbb{D}$ & A subgroup dataset \\
			$F^S=F\setminus F^{\prime} \subset F$ & candidates for feature selection \\
			${\mathcal F}$ & Powerset of $F^S$ \\
			$F_i^-, F_i^+$ & Systematically missing features and \\
            & complementary subset\\
			${\mathcal F_i^-}, {\mathcal F_i^+}$ & Feature subset of subgroups with (without)\\
			& systematically missing features\\
			$F^m=\{f_1,f_2,\dots,f_m\} \subset F^S$ & Feature subset of size $m$ \\
			$\mathbb{D}^m$ & Vertical subset of $\mathbb{D}$ \\
			$F_i^{m*}$ & Feature subset with highest MI for $\mathbb{D}_i$ \\
			$F_i^{mj}$ & $j$-th top feature subset with $\mathbb{D}_i$'s highest MI \\
			$TopK^m_i$ & $m$-size top-$K$ feature subsets of subgroup $\mathbb{D}_i$ \\
			$B_i$ & budget for subgroup $\mathbb{D}_i$ \\
			$\mathbb{G}_i=\left( \mathbb{V}_i, \mathbb{E}_i\right),\mathbb{G}=\left( \mathbb{V}, \mathbb{E}\right)$ & Subgroup feature lattice graph,\\ & multiple lattice graph\\
			$level_{min},level_{max}$ & lower and upper bounds of sampling and\\
			& prediction levels over $\mathbb{G}$\\
			\hline
		\end{tabular}
		\caption{Table of notations}
		\label{tab:symbol}
	\end{table}
\end{footnotesize}

We conclude with introducing the {\em upward closure of MI} property, which becomes handy in addressing the computational challenges of this work. This property
guarantees that the MI of a feature set $F_1$ is never larger than any of its super-set $F_2$ ($F_1\subseteq F_2$)~\cite{salam2019human, brin1997beyond}:
\begin{equation}
	I(F_1; Y) \leq I(F_2; Y)
\end{equation}

\subsection{Graph Neural Networks}\label{sec:GNNs}
Graphs serve as an effective tool for capturing relationships and interactions among data objects.
GNNs leverage graph structures, serving as a complementary inference framework. They provide advanced capabilities for performing tasks, including node, link, and graph-level predictions, on such graphs.
The fundamental premise of GNNs is the inherent interconnectedness of data objects, enabling information from one entity to influence another~\cite{hamilton2020graph,kipf2016semi,scarselli2008graph, zhou2020graph}.


Nodes in a GNN are instantiated with initial feature vectors by assigning attributes or embedding representations based on node properties or external knowledge. These vectors serve as a starting point for the iterative neural message passing mechanism of the GNN, propagating information across the graph. In each iteration, every node receives vector messages transmitted by its direct neighbors and aggregates them ({\em e.g.}, by using summation, mean, or max pooling) to form a new personal message (hidden representation) for itself. This updated message is sent to neighboring nodes in the subsequent iteration. Such message passing iteration occurs concurrently for all nodes, attributed as a {\em single GNN layer}.

A general formulation for updating the hidden representation of a node $u$ through message passing in a heterogeneous graph structure can be described as follows.
\begin{equation}
	\label{eq:multi_edge_gnn_layer}
	\boldsymbol{h}_u^t = \text{AGGREGATE}\left(\boldsymbol{h}_u^{t-1}, \left\lbrace\left\lbrace \boldsymbol{h}_v^{t-1} \mid v \in \mathcal{N}_r(u) \right\rbrace \forall r \in R \right\rbrace \right)
\end{equation}
\noindent where $u$ is the node of interest, $t$ is the iteration (GNN layer) number, $R$ is a set of node types, $N_{r}(u)$ is node $u$'s neighborhood with respect to node type $r \in R$, $\boldsymbol{h}^{t-1}_v$ is the previous layer hidden representation of the neighbor $v$, and AGGREGATE combines $u$'s previous hidden representation with the representations of its neighbors, incorporating mathematical operations and learnable parameter matrix.

The overall architecture of a GNN consists of several layers stacked together. The output of the last layer consists of the latent representations of nodes (embeddings). 

In a heterogeneous graph (see Section~\ref{sec:algorithm}), with multiple node and edge types, conveying information from neighbors may be differ depending on their type~\cite{schlichtkrull2018modeling}. 
The model captures the diverse structural characteristics of the graph by adaptively aggregating information from different neighbor types. 

Typically, message passing and embeddings generation are performed over all nodes in the graph. For selecting $F^m$, we demonstrate a setting where embeddings are generated only for a subset of the graph, based on the graph structure (Section~\ref{sec:graphConstruction}) and application needs (sections~\ref{subsection:training} and~\ref{subsection:topk}).

\section{Data Model and Problem Definition}
\label{sec:problem_definition}
\ignore{do we need this? reads verbose and it is not even clear what is prediction, MI computation, etc. IMHO, this whole paragraph could be hidden. We now provide the necessary building blocks for our work, and provide a formal specification of the problem of feature selection for multiple subgroups in the presence of systematic missing data, by utilizing MI as a mean to asses feature combination importance. In a nutshell, we aim at striking a balance between the number of MI computations and predictions, where the latter is performed by solving a machine learning prediction task. Clearly, for a feature subset that contains a feature with systematic missing data, we can only resort to prediction. For the remaining feature subsets, we are left with the decision to either compute or predict, trading off cost of explicit computation with accuracy of prediction.}


Let $F^{\prime}\subset F$ be a subset of features. We create subgroups of the dataset using $F^{\prime}$ by replacing $\mathbb{D}$ (Section~\ref{subsection:FeatureSelection}) with $\mathbb{D} = \bigcup_{i=1}^{|P|} \mathbb{D}_i$, where $P$ is a complete set of minterm (conjunction of simple and negated simple) predicates over $\mathbb{D}$, $p_i$ is a conjunctive predicate over the set of features $F^{\prime}$, and $\mathbb{D}_i=\sigma_{p_i}\mathbb{D}$ is a selection over $\mathbb{D}$ according to $p_i$ ($1\leq i\leq |P|$). 
Example~\ref{ex:Readmission_subgroups} uses predicates such as ``Age$<$40 AND Ethnicity=`Asian'.'' 
The feature selection task is performed over ${\mathcal F}=2^{F\setminus F^{\prime}}$, the entire set of feature combinations, excluding the subgrouping features. We use $F^S=F\setminus F^{\prime}$ to denote the feature subset that is used for the feature selection process. 

Feature subsets, elements of ${\mathcal F}$, can be naturally organized in a lattice~\cite{salam2019human}, a $|F^S|$-dimensional hypercube, where nodes represent feature subsets and an edge exists between two nodes with hamming distance of $1$. Subsets in the lattice are arranged in a hierarchical manner based on the number of features included in a combination. The lowest level of the lattice consists of singleton nodes, each representing a single feature. Moving up to the next level, each node represents a pair of features built upon the singletons below. These pairs are formed by considering all possible combinations of two features. This pattern continues for higher levels of the lattice, with each level representing feature combinations of increasing size. As we move up the lattice, the combinations become more complex, incorporating more features from the dataset.
\begin{definition}
Given a subgroup data fragment $\mathbb{D}_i$, we say that feature $f\in F^S$ is {\em systematically missing} in $\mathbb{D}_i$ if 
\begin{equation}
	\forall t\in \mathbb{D}_i, t[f]=NULL
\end{equation} 
\end{definition}
A feature may be systematically missing in one subgroup but not in others. In Example~\ref{ex:Readmission_subgroups}, feature ``Cholesterol'' is systematically missing only for subgroups defined by predicate ``Age$<$40''. 
We denote by $F_i^-$ and ${\mathcal F_i^-}$ the set of systematically missing features and feature subsets that contain one or more systematically missing features for subgroup $\mathbb{D}_i$, respectively. $F_i^+$ and ${\mathcal F_i^+}$ are the complementary set, such that $F^S=F_i^-\cup F_i^+$ and ${\mathcal F}={\mathcal F_i^-}\cup {\mathcal F_i^+}$.


For a subgroup $\mathbb{D}_i$ and a subset of features $\tilde{F}\in {\mathcal F}$, $MI_i^{\tilde{F}}$ is the empirical MI value of $\tilde{F},Y$ with respect to $\mathbb{D}_i$, where $Y$ is the target variable ({\em e.g.}, ``Readmission" in Example~\ref{ex:Readmission_subgroups}). We restrict the selection process to a subset of features of a fixed size $m<|F^S|$, justified by computational complexity of MI for large feature sets. Let ${\mathcal F}^m=\{F^m\in{\mathcal F}\mid |F^m|=m\}$ denote the set of all subset features of size $m$ (level $m$ of the lattice with $|{\mathcal F}^m|=\binom{|F^S|}{m}$ nodes) and $F_i^{m*}$ denote the feature subset that returns the highest empirical MI value for subgroup $\mathbb{D}_i$ of all feature subsets of size $m$. Therefore, 

\begin{equation}\label{eq:bestMI}
	F_i^{m*}=\argmax_{F^m\in {\mathcal F}^m}MI_i^{F^m}
\end{equation}
for a subgroup $\mathbb{D}_i$.
We denote $F_i^{m*}$ over ${\mathcal F}$ by $F_i^{m1}$ and recursively define $F_i^{mj}$ (the $j$-th top feature subset). 
\begin{definition}
	The {\em $m$-size top-$K$ feature subsets of subgroup $\mathbb{D}_i$} $TopK^m_i=\{F_i^{m1},F_i^{m2},\dots,F_i^{mK}\}$ is defined recursively as follows.
	\begin{equation}
			F_i^{m1}=\argmax_{F^m\in {\mathcal F}^m}MI_i^{F^m}
	\end{equation}
	For $1<j\leq K$
	\begin{equation}
			F_i^j=\argmax_{F^m\in {\mathcal F}^m\setminus (\cup_{l=1}^{j-1}F_i^{ml})}MI_i^{F^m}
	\end{equation}
\end{definition}

\begin{problem}
	\label{problem:top_k_fs}
	Given a feature set of size $m$ and an integer $k$, for each subgroup $\mathbb{D}_i$, return $TopK^m_i=\{F_i^{m1},F_i^{m2},\dots,F_i^{mK}\}$.
\end{problem}
\begin{sloppypar}
We wish to identify Top-$K$ sets, each with $m$ features, for each subgroup $\mathbb{D}_i$. Using Examples~\ref{ex:Readmission_subgroups} and~\ref{ex2}, if $k=2$ and $m=3$, then the goal is to find top-$2$ sets, each with $3$ features for the $4$ different subgroups presented in Example~\ref{ex:Readmission_subgroups}.
\end{sloppypar}

A na\"{\i}ve solution to Problem~\ref{problem:top_k_fs} involves enumerating MI values of $\binom{|F^S|}{m}$ feature subsets of size $m$, and sort them for the top-$K$ subsets of each subgroup. 
Such a computation is exponential in $|F^S|$ and therefore quickly becomes computationally expensive, as the number of {\bf candidate} features increases, regardless of $m$. Moreover, whenever $TopK^m_i\cap {\mathcal F_i^-}\not=\emptyset$, we run into a problem of directly computing the MI of feature subsets due to systematic missing data. Therefore, we are forced 
to predict rather than compute the MI value of feature subsets that consist of features with systematic missing data. To summarize, systematic missing data, combined with an exponential search space, guide us towards reducing the number of MI computations, replacing computation with prediction even if a feature subset can be computed directly. 




\section{Never Miss a Feature with \alg}
\label{sec:algorithm}

Equipped with a data model, we introduce \alg, an efficient solution to Problem~\ref{problem:top_k_fs}. We aim at retrieving $TopK_i^m$ through a hybrid approach that involves training a model to predict MI of feature subsets. 
We study this as a graph representation learning problem, framing it as an MI prediction (regression) task to estimate MI scores for feature subsets. 

\algspace leverages a GNN to propagate information over a graph and capture the domain inherent structure and constraints (Section~\ref{sec:graphConstruction}). The algorithm is performed in three main steps. First, we introduce a sampling mechanism over the graph structure, to reduce the MI computation complexity (Section~\ref{sec:sampling}). Once the MI values  of the sampled nodes are computed, \algspace moves to the training phase, with the outcome of 
a model that captures the relationships between feature subsets in and among subgroups (Section~\ref{subsection:training}). The training phase is followed by $TopK^m_i$ computation for all subgroups, a solution to Problem~\ref{problem:top_k_fs} (Section~\ref{subsection:topk}).

\subsection{Feature Lattice Graph Construction}
\label{sec:graphConstruction}
A single lattice graph encapsulates all feature subsets per subgroup, and their dependencies. 
Each feature subset within a subgroup is represented as a node in a graph, such that all feature subsets of a subgroup constitutes the set of nodes.  The lattice provides a systematic way to explore the space of feature subsets, starting from individual features and gradually building up to larger subsets. The edges of this hierarchical structure aids in capturing the upward closure property of MI~\cite{salam2019human} by depicting interrelations between feature subsets. To facilitate information propagation across different subgroups, we connect different lattices. We employ a heterogeneous GNN trained on a node prediction task, specifically predicting the MI score of a feature subset. By leveraging this multiple lattice graph structure, our model captures latent dependencies between feature subsets, both within and across subgroups.


We begin with describing the construction of a single subgroup feature lattice graph (Section~\ref{subsection:single_lattice}), and continue with the generation of a multiple lattice graph (Section~\ref{subsection:multiple_lattice_graph}). 


\subsubsection{\textbf{Subgroup Feature Lattice Graph}}
\label{subsection:single_lattice}
In what follows, we use $F_1$ and $F_2$ to denote two feature subsets in ${\mathcal F}$.
\begin{definition}[Subgroup feature lattice graph] \label{def:SFLGraph} 
	Given a subgroup $\mathbb{D}_i$, the {\em subgroup feature lattice graph} is an undirected graph $\mathbb{G}_i=\left( \mathbb{V}_i, \mathbb{E}_i\right)$, where 
	\begin{compactitem}
		\item (nodes:) $\mathbb{V}_i={\mathcal F}\setminus\emptyset$
		\item (edges:) $\mathbb{E}_{i}=\mathbb{E}^{Inter}_{i}\cup\mathbb{E}^{Intra}_{i}$ such that
		\begin{compactitem}
			\item (inter-level edges:) $(v_i,v_j)\in\mathbb{E}^{Inter}_{i}$ if (1) $\{v_i,v_j\}\subset\mathbb{V}_i$; (2) $v_i=F_1$; (3) $v_j=F_2$; (4)  $|F_1|-|F_2|=1$; and (5) $F_1\subset F_2$
			\item (intra-level edges:) $(v_i,v_j)\in\mathbb{E}^{Intra}_{i}$ if (1) $\{v_i,v_j\}\subset\mathbb{V}_i$; (2) $v_i=F_1$; (3) $v_j=F_2$; (4) $|F_1|=|F_2|=\ell$; (5) $|F_1\cap F_2|=\ell-1$; and (6) $|F_1\cap F_2|>0$
		\end{compactitem}
	\end{compactitem}
\end{definition}
We initiate a $1:1$ 
mapping between feature combinations and their respective representations using binary encoding. In this scheme, each element within the binary vector corresponds to a distinct feature in the dataset. When a feature is included in a particular combination, the corresponding element in the binary vector is assigned with a value of $1$; otherwise, it retains a value of $0$. With each element in the vector aligned to a specific feature, the dimensionality of these vectors is 
$|F^S|$.

This binary encoding offers a concise and informative feature combinations representation, facilitating efficient processing within the framework of GNN. It establishes an indispensable association between feature combinations and the hierarchical levels within the lattice. At each level, the number of $1$'s in a vector mirrors the level number, reflecting the hierarchical nature of the lattice and its alignment with the encoded feature combinations.

\begin{example}
\label{ex:lattice_graph_nodes}
Figure~\ref{fig:lattice_graph} illustrates a lattice of the feature set: $F^S=\left\lbrace f_0, f_1, f_2, f_3 \right\rbrace$.
At the first level of the lattice, each binary representation contains a single $1$ element, corresponding to the underlying feature that forms the singleton combination, while all other elements are $0$. For example, the binary vector $0001$ represents the combination $\left\lbrace f_0\right\rbrace $, and $0100$ represents $\left\lbrace f_2\right\rbrace $.
Moving up to the second level, each feature combination is now denoted by a vector with two 1's. For example, the combination $\left\lbrace f_0, f_2 \right\rbrace $ is encoded by the vector $0101$. On top of the lattice lies the vector $1111$ that represents the feature combination of all available features.

\begin{figure}[htpb]
	\centering
	\includegraphics[width=0.9995\columnwidth]{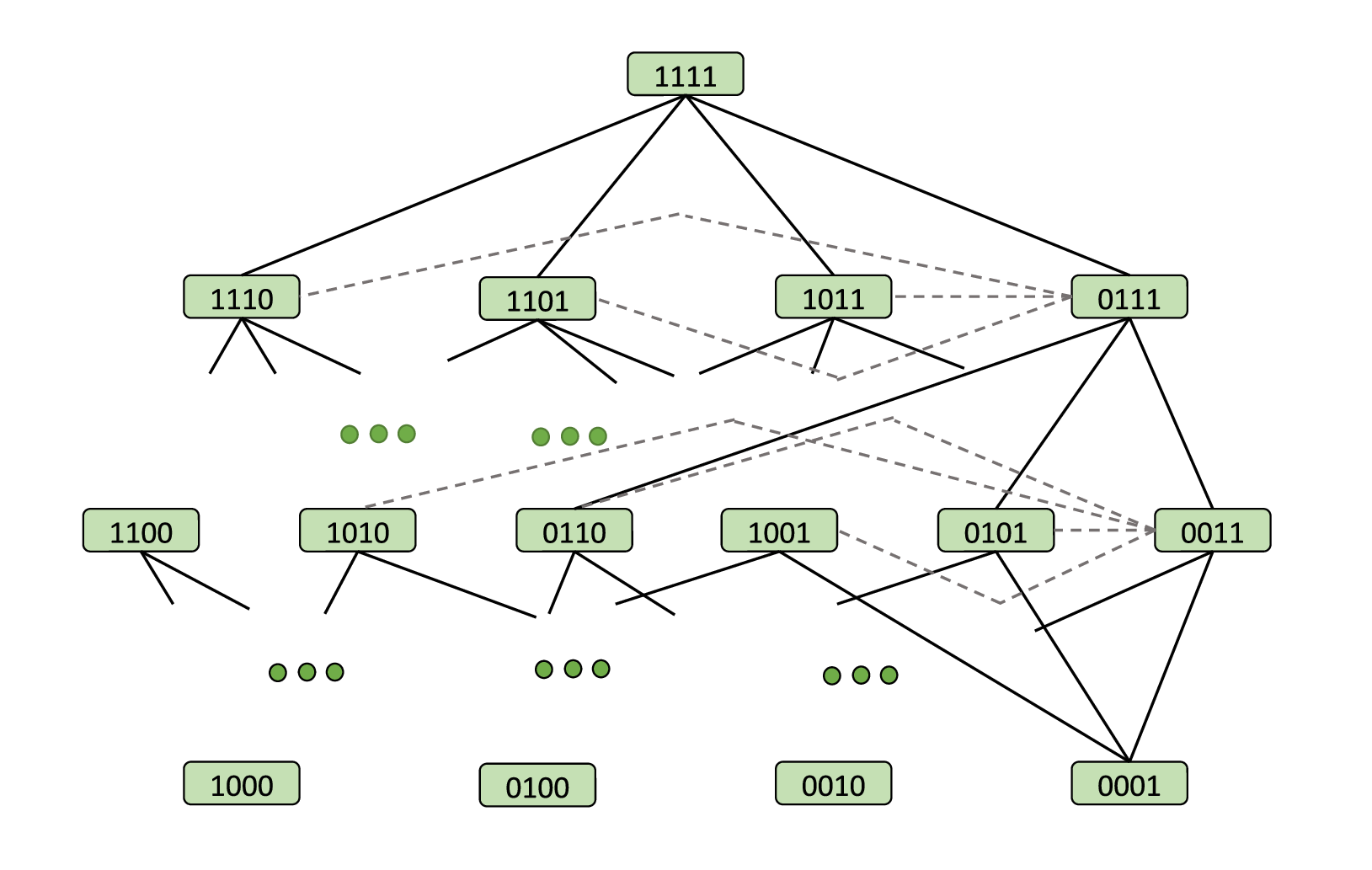}
	\caption{An Illustration of a single, four level lattice graph.}
	\label{fig:lattice_graph}
\ifTechReport
\vspace*{-0.25cm}
\fi
\end{figure}
\end{example}

The lattice hierarchical structure is encoded in the subgroup feature lattice graph as {\em inter-level edges}. These edges, connecting subsumed feature subsets with cardinality difference of $1$, ensure the learning abides by the upward closure of MI property (Section~\ref{sec:Preliminaries}). 

In addition, the graph contains {\em intra-level edges}, connecting nodes in the same level of the lattice (cardinality difference of $0$). These edges ensure that when the model learns about a feature subset, then other feature subsets that overlap with the former set benefit from this learning. We establish intra-level edges between feature subsets that differ by a single feature. 
\begin{example}
\label{ex:lattice_graph_edges}
Consider, again, Figure~\ref{fig:lattice_graph}. The gray dashed lines form a partial set of the intra-level edges.  For example, the node represented by the vector $0011$ is connected to all other feature combinations of size $2$ that include either $f_0$ or $f_1$: $0101, 1001, 0110, 1010$. The black solid lines are the inter-level edges. For illustration of the subsumption-based connection scheme, consider the node $0111$, representing the feature combination $f_0, f_1, f_2$. This feature combination subsumes three feature combinations of size $2$: $\left\lbrace f_0, f_1\right\rbrace, \left\lbrace f_0, f_2\right\rbrace, \left\lbrace f_1, f_2\right\rbrace$. Hence, the node $0111$ is connected by a solid line to their corresponding nodes, represented by the vectors $0011$, $0101$ and $0110$, respectively.
\end{example}

\begin{lemma}\label{lemma:exp-size}
The size of the subgroup feature lattice graph $\mathbb{G}_i=\left( \mathbb{V}_i, \mathbb{E}_i\right)$ is exponential to the number of features both in terms of the number of nodes $|\mathbb{V}_i|$ and the number of edges $|\mathbb{E}_i|$.
\end{lemma}



\ifTechReport
\begin{proof}
\sloppy
Recall that each level in the lattice graph corresponds to the size of feature combinations reside in it, ranging from individual features to the full set. Specifically, at level $\ell$, the number of nodes is determined by the binomial coefficient, $\binom{|F^S|}{\ell}$. 
Summing 
across all levels yields $2^{|F^S|} - 1$ nodes, which is exponential to $|F^S|$.

An inter-level edge is created based on subsumption. A node with $\ell$ features is connected to each of its proper subsets of size $\ell-1$. The number of such combinations is $\binom{\ell}{\ell-1}=\ell$. Therefore, the total number of inter-level nodes is $\sum_{\ell=2}^{|F^S|} \binom{|F^S|}{\ell} \cdot \ell$. Note that the summation starts from $2$ since nodes of the first level are not connected by an inter-level edge to any other subsumed node by Definition~\ref{def:SFLGraph}. 
To reach a closed-form expression, we recall that
$\left( 1+x\right)^{|F^S|} = \sum_{\ell=0}^{|F^S|}\binom{|F^S|}{\ell} \cdot x^{\ell}$. Differentiating both sides w.r.t. $x$ yields: $|F^S| \cdot \left( 1+x\right)^{|F^S|-1} = \sum_{\ell=0}^{|F^S|}\binom{|F^S|}{\ell} \cdot \ell \cdot x^{\ell - 1}$. After setting $x=1$ we get: $|F^S| \cdot  2^{|F^S|-1} = \sum_{\ell=0}^{|F^S|}\binom{|F^S|}{\ell} \cdot \ell$. We note that the term with $k=0$ does not contribute to the sum because it evaluates to $0$. Since we are interested in the sum beginning from $2$ we subtract the first element in the sum ($|F^S|$), obtained with $\ell=1$, from both sides and get: $|F^S| \cdot 2^{|F^S|-1} - |F^S| = \sum_{\ell=2}^{|F^S|}\binom{|F^S|}{\ell} \cdot \ell$. The right-hand side is exactly the expression describing the number of inter-level edges. The left-hand side can be simplified and rewritten as $\frac{|F^S|}{2} \left( 2^{|F^S|} - 2\right)$, which is exponential in $|F^S|$.

The number of intra-level edges is determined based on overlap. A node in $\ell$, representing a certain combination, is connected to all other nodes within its level that share $\ell-1$ common features with it, and differ only by a single feature. Every node of interest in level $\ell$, contains $\binom{\ell}{\ell-1}=\ell$ combinations of feature subsets, each of size $\ell-1$. Corresponding to each such combination of size $\ell-1$, there are $|F^S|-\ell$ nodes that differ from the node of interest by a single feature.
This logic dictates the following formula for the number of intra-level edges:  $\sum_{\ell=2}^{|F^S|} \binom{|F^S|}{\ell} \cdot \frac{\ell \cdot \left( |F^S|-\ell\right)}{2}$, wherein the divisor $2$ ensuring that each edge is not counted twice.
Thus, the number of intra-level edges is also exponential to $|F^s|$.
\end{proof}
\fi

The proof of Lemma~\ref{lemma:exp-size}, 
offers an exact computation for the number of nodes and edges in a subgroup feature lattice graph, beyond substantiating the exponential nature of the graph. Our empirical evaluation (Section~\ref{sec:experiments}) shows that the number of edges has a low impact on computation, as long as the overall graph size can be stored in memory as a whole. However, the number of nodes, and in particular the need to compute MI values for a large number of nodes, has a significant impact on the overall algorithmic solution performance. Consequently, we opt to limit the number of layers over which the proposed algorithm iterates (Section~\ref{subsection:training}) and to sample nodes for MI computation (Section~\ref{sec:sampling}).

\subsubsection{\textbf{Multiple Lattice Graph}}
\label{subsection:multiple_lattice_graph}
The subgroup feature lattice graph, $\mathbb{G}_i$, forms a key component in our learning framework, enabling flexibility in learning and prediction by treating identical feature subsets based on the contextual differences (subgroups). 
Next, we define a multiple lattice graph, generated by interconnecting the subgroup feature lattice graphs. The resulting structure is a multiplex graph, a special type of heterogeneous graph, 
allowing for the representation of a concept (in our case, feature subset) in different contexts (subgroups).
Our multiplex graph incorporates all nodes and edges from the subgroup feature lattice graphs, complemented by a collection of inter-lattice edges.
\begin{definition}[Multiple lattice graph]\label{def:MLG}  $\mathbb{G}=\left( \mathbb{V}, \mathbb{E}\right)$, where:
\begin{compactitem}
	\item (nodes:) $\mathbb{V} = \bigcup_{i=1}^{|P|} \mathbb{V}_i$
	\item (edges:) $\mathbb{E} = \left( \bigcup_{i=1}^{|P|} \mathbb{E}_i\right) \cup \left( \bigcup_{\substack{i,j=1 \\ i \neq j}}^{|P|} \mathbb{E}_{i,j}\right)$ such that $(v_1,v_j)\in\mathbb{E}_{i,j}$ if (1) $v_i\in\mathbb{V}_i$; (2) $v_j\in\mathbb{V}_j$; (3) $v_i=F_1$; (4) $v_j=F_2$; and (5) $F_1=F_2$ 
\end{compactitem}  
\end{definition}


Pairs of lattices are connected by linking nodes of the same feature subsets. Therefore, between any two subgroup feature lattice graphs (and we have $\binom{|P|}{2}$ such pairs) we create at most $|{\mathcal F}| - 1$ edges. This approach ensures a comprehensive integration of information between diverse subgroups.


$\mathbb{G}$ is a heterogeneous graph, representing a comprehensive information integration between diverse subgroups. The graph encompasses various types of edges, reflecting the different relationships between subgroups. The inter-lattice edges correspond to connections between nodes from different subgroup pairs. It is important to note that each subgroup pair results in a unique edge type within this set. Conversely, the in-lattice edges signify relationships within individual subgroups, with each subgroup inducing its distinct edge type within this set.

	\vskip-.1in
\begin{figure}[htpb]
	\centering
	\includegraphics[width=0.9995\columnwidth]{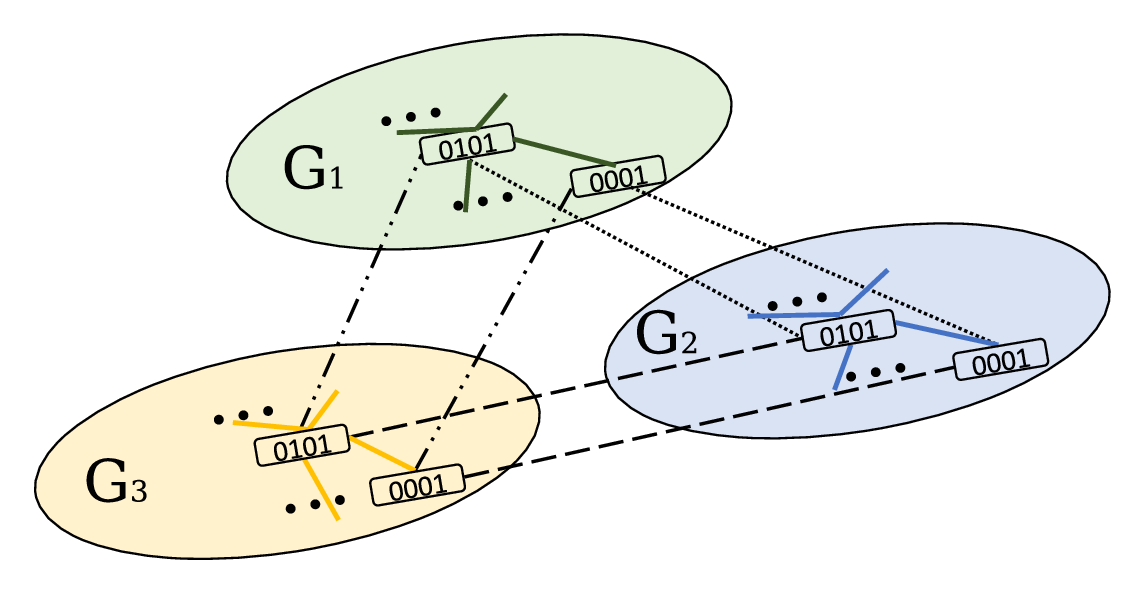}
	\vskip-.1in
	\caption{An Illustration of a multiple lattice graph structure.}
	\vskip-.1in
	\label{fig:multiple_lattice_graph}
\end{figure}

\begin{example}
	\label{ex:multiple_lattice_graph}
Figure~\ref{fig:multiple_lattice_graph} illustrates a heterogeneous graph constructed from three subgroups, each represented by a separate lattice graph ($\mathbb{G}_1, \mathbb{G}_2$, and $\mathbb{G}_3$). In this example, we consider a simplified scenario where only two nodes, namely 0001 and 0101, are depicted for each lattice. The interconnections between the lattices signify the relationships between different subgroups. Each connection is denoted by a different edge shape, reflecting the diversity among edge types, which varies based on the subgroup pairs involved. A node is directly connected to all other nodes representing the same feature combination across the remaining lattice graphs. Solid lines characterize the in-lattice edges, with each line uniquely colored to denote its lattice origin, thus illustrating the distinct edge types represented within each lattice.
\end{example}



\subsection{Efficiency Opportunities}
\label{sec:sampling}
Recall 
that the size of the subgroup feature lattice graph $\mathbb{G}_{i}=\left( \mathbb{V}_{i}, \mathbb{E}_{i}\right)$ as well as the multiple lattice graph $\mathbb{G}=\left( \mathbb{V}, \mathbb{E}\right) $, are exponential to the feature set size (Lemma~\ref{lemma:exp-size}). 
Even though done once, training such a gigantic heterogeneous network is prohibitively expensive both computationally and storage space-wise. We study two distinct efficiency opportunities in the training process, as discussed next.

\subsubsection{Pre-computing and Sharing Computation of MI in $\mathbb{G}_{i}$}
\label{sec:MI-sharing}
As Figure~\ref{fig:lattice_graph} illustrates, each subgroup feature lattice graph represents the power set of features, each being a node. A na\"ive implementation 
involves running an exponential number of computations, where for a node with $m$ features there is a need for a nested loop of size $m+1$ for computing MI with the target variable (Eq.~\ref{eq:MI}). 

We observe that for each $\mathbb{G}_{i}$, inter-level edges connect nodes that are related through subsumption. For example, node represented by the vector $0011$ is a superset of nodes $0001$ and $0010$. 
We therefore argue that the na\"ive implementation contains redundant computations that can be eliminated by exploiting relationship between joint entropy and MI and the chain rule of joint entropy~\cite{cover1991entropy}. 

MI of a set of $x$ random variables (in our case, $x-1$ features and one target variable) could be represented as follows, where $H(\cdot)$ represents the joint entropy~\cite{cover1991entropy} of the random variables involved in the calculation.
	\begin{equation}\label{eq:joint}
	\begin{split}
		I(A_1,A_2,\ldots,A_x)  = &\Sigma_{i=1}^x H(A_i) - \Sigma_{1\leq i <j \leq x}H(A_i,A_j)+ \\
		& \Sigma_{1\leq i <j < k \leq x}H(A_i,A_j,A_k)+ \ldots +\\
		& (-1)^{x+1}H(A_1,A_2,\ldots,A_x)
	\end{split}
	\end{equation}
Basically, Eq.~\ref{eq:joint} demonstrates that the MI of $x$ random variables could be expressed as a linear relationship of joint entropy involving all possible subsets of these $x$ variables.
Then, the chain rule of joint entropy states
	\begin{equation}\label{eq:chainrule}
	\begin{split}
		H(A_1,A_2,\ldots,A_x)= & H(A_1,A_2,\ldots,A_{x-1})+\\
		&H(A_x|A_1,A_2,\ldots,A_{x-1})
	\end{split}
	\end{equation}
 Using Eq.~\ref{eq:chainrule}, $H(A_1,A_2)=H(A_1)+H(A_2|A_1)$. The generalized form says that the joint entropy  of a set of $x$ random variables could be computed by simply adding  two constructs: a) joint entropy of any of those $x-1$ random variables, and b) the conditional entropy of the remaining extra variable, conditioned on the same subset of $x-1$ variables. The order of the random variables does not matter in this process.

Given $x$ features, we pre-compute a set of $2^x$ constructs of joint and conditional entropy: $x$ constructs for $x$ entropy values of the individual features, $\binom{x}{2}$ conditional entropy constructs  needed to compute joint entropy of all size $2$ feature set, {\em e.g.}, $H(A_2|A1), H(A_3|A_1),\ldots, H(A_x|A_1)$, because $H(A_iA_j)=H(A_i)+H(A_j|A_i)$, $\binom{x}{3}$ conditional entropy constructs needed to compute joint entropy of all size $3$ feature set, {\em e.g.}, $H(A_3|A_1,A_2)$,  $\ldots$, $\binom{x}{x-1}$ conditional entropy constructs needed to compute joint entropy of all size $(x-1)$ feature set. The training process simply adds and subtracts (Eq.~\ref{eq:joint}) these pre-computed constructs to compute MI of any feature set with the target variable. 

To understand the extent of computational saving, consider some level $i+1$ of the lattice. To compute MI of a feature set of size $i+1$, we reuse pre-computed constructs of conditional entropy of size $i$, adding to it the joint entropy of feature set of size $i$ (which was computed in the previous step) using a single addition operation. This is in contrast to a na\"ive computation that requires repeated computation of conditional entropy. 

\subsubsection{Sampling Subgroup Feature Lattice Graph} 
Given a subgroup feature lattice graph $\mathbb{G}_{i}=\left( \mathbb{V}_{i}, \mathbb{E}_{i}\right)$, we next study how to identify a set $\mathbb{\tilde{V}}_i \subset \mathbb{V}_{i}$ of nodes of size $B_i$ (a budget parameter). 
With infinite computing resources and capabilities, the different feature subsets present in $\mathbb{V}_i$ would have produced a distribution of MI ($PDF_{MI}(\mathbb{V}_i)$) with the target variable $Y$. Ideally, if computation and storage were not bottlenecks, one would retain the whole $\mathbb{G}_{i}=\left( \mathbb{V}_{i}, \mathbb{E}_{i}\right)$. Under a budgetary constraint, one natural goal of sampling is thus to retain those $\mathbb{\tilde{V}}_i$ such that if one creates a MI distribution using $\mathbb{\tilde{V}}_i$ that distribution should be as close as possible to the MI distribution of $\mathbb{G}_{i}=\left( \mathbb{V}_{i}, \mathbb{E}_{i}\right)$. However, computing the MI distribution of $\mathbb{G}_{i}=\left( \mathbb{V}_{i}, \mathbb{E}_{i}\right)$ is infeasible in the first place. We therefore wish to generate a set of sampled nodes such that the distribution of their MIs ($PDF_{MI}(\mathbb{\tilde{V}}_i)$) is a uniform representative of $\mathbb{G}_{i}$ and design a highly efficient solution that gives theoretical guarantees towards that.

   \begin{problem}[Sampled subgraph node selection]
\label{problem:sampling size}
	Given a budget $B_i$ for subgroup $\mathbb{D}_i$ 
    identify $\mathbb{\tilde{V}}_i$ of size $B_i$ such that $PDF_{MI}(\mathbb{\tilde{V}}_i)$ is a uniform random sample of $PDF_{MI}(\mathbb{V}_i)$.
\end{problem}
\begin{algorithm}
	\caption{Sampling Algorithm \algsample}\label{alg:samp}
	\begin{algorithmic}[1]
		\Require $B_i$
		\State $\mathbb{\tilde{V}}_i=v_1$ /* $v_1$ is generated by setting each $j \in F_i^+$ to $1$ or $0$ uniformly at random and independently of all others. All bits in $F_i^-$ are assigned to $0$  */\label{alg:initial}
		\State $c=2$
		\While{$c \leq B_i$} \label{alg:round}
		\State With probability $1/2$, goto~\ref{alg:increase} \label{alg:stay}
		\State With probability $1/2$, create $v_c$ from $v_{c-1}$ by  \label{alg:move1}\\
		\quad \quad flipping the $j$-th bit, chosen uniformly from $F_i^+$ \label{alg:move1a}
		\State $\mathbb{\tilde{V}}_i=\mathbb{\tilde{V}}_i\cup v_c$ \label{alg:move2}
		\State If $|\mathbb{\tilde{V}}_i|=c$ then $c=c+1$ \label{alg:increase}
		\EndWhile
	\end{algorithmic}
\end{algorithm}

 \algsample (Algorithm~\ref{alg:samp}) 
 samples using a lazy random walk~\cite{levin2017markov} on the feature 
 $|F^S|$-dimensional hypercube, a walk that starts at a random node in the hypercube that is computable, that is a node that does not contain features with systematically missing data. 
 Recall the binary encoding of a feature subset (Section~\ref{subsection:single_lattice}). We generate an initial node representation by setting  each of the $F_i^+$ feature values to $1$ or $0$ uniformly at random and independently of all others (Line~\ref{alg:initial}). Then, the random walk works iteratively until the budget is consumed. With $1/2$ probability it stays at the same node (and the same level) of the hypercube 
 as it is currently (Line~\ref{alg:stay}). This step is needed to avoid cycles and to ensure convergence. With probability $1/2$, it chooses a feature in $F_i^+$ and flips its current value, allowing the walk to go to one of the neighboring computable nodes, moving either one level up or one level down (lines~\ref{alg:move1} and~\ref{alg:move1a}). If this subset 
 has not been seen before, it is added to the sample (Line~\ref{alg:move2}). The process ends when $B_i$ samples are collected. 
 



\algsample takes exactly $|F^S|$ time to decide the first sample. 
Then, the choice of finding each subsequent sample takes constant time, but it does not guarantee a new 
 subset. Therefore, running time of \algsample can be bounded from below by $\Omega(|F^S|+B_i)$.


\begin{lemma}
   \algsample produces a uniform random sample of $PDF_{MI}(\mathbb{V}_i)$. 
\end{lemma}

\begin{proof}
    (Sketch.) It could be shown that \algsample induces a Markov chain that is aperiodic and irreducible~\cite{mc1,mc2}, converging to its unique stationary distribution, which is known to produce a uniform distribution over the $|F^S|$  dimensional hypercube~\cite{mc1,mc2}. \algsample hence produces a uniform random sample. 
\end{proof}


\subsection{Training a Model using $\mathbb{G}$}\label{subsection:training}
Each subgroup feature lattice graph $\mathbb{G}_i$ in the multiplex graph $\mathbb{G}$ dictates a different node type (see Section~\ref{subsection:multiple_lattice_graph}).
This heterogeneity is reflected in our message passing scheme wherein distinct parameter matrices are learned for distinct node type pairs, allowing for specialized information propagation within the multiplex graph.

We adapt GraphSage~\cite{hamilton2017inductive} for heterogeneous graphs. 
The aggregated neighborhood message of $v_{i} \in \mathbb{V}_i$ at layer $t$ is 
\begin{equation}
	\label{eq:graphsage_neighbors}
	\boldsymbol{h}^t_{N \left(v_{i} \right)} = \left( \frac{\boldsymbol{W}^t_{i}}{|N_{i}\left(v_{i}\right)|} \sum_{u_i \in N_{i}\left(v_{i}\right)} \boldsymbol{h}_{u_i}^{t-1} + \sum_{\substack{j=1 \\ j \neq i}}^{|P|} \boldsymbol{W}^t_{j,i} \boldsymbol{h}_{v_{j}}^{t-1}\right)
\end{equation}
where $N_{i}\left(v_{i}\right)$ denotes the neighboring nodes of $v_{i}$ within the same subgroup lattice graph $\mathbb{G}_i$. Both $\boldsymbol{W}^t_{i}$ and $\boldsymbol{W}^t_{j,i}$ are learnable weight matrices of the $t$-th GNN layer. The former refers to message aggregation from neighbors within $\mathbb{G}_i$, while the latter to message aggregation from a corresponding node $v_{j}$ of $\mathbb{G}_j$.

To yield the final node representation of $v_{i}$ at layer $t$, the aggregated neighborhood message is concatenated to the representation of $v_i$ from the previous layer, and the resulted concatenated vector is then multiplied with another trainable weight matrix as follows:
\begin{equation}
	\label{eq:graphsage_node_representation}
	\boldsymbol{h}^t_{v_{i}} = \sigma \left( \boldsymbol{W}^t_{conc} \cdot \left[\boldsymbol{h}^{t-1}_{v_{i}}\ ||\ \boldsymbol{h}^t_{N \left(v_{i} \right)} \right] \right)
\end{equation}
where $\sigma$ is an activation function (\emph{e.g.,} ReLU), $||$ is concatenation operator and $\boldsymbol{W}^t_{conc}$ is a fully connected layer.

The final node representation, derived after a predefined number of message passing iterations, 
is fed into a regression head, implemented as a fully connected neural network with \emph{Mean Squared Error} (MSE) loss function, which predicts node MI scores.

Unlike inductive learning, where the model is trained on a subset of nodes and is expected to generalize to unseen nodes in the same graph or similar graphs, \alg's transductive learning involves leveraging information from the entire graph (possibly restricted, for efficiency purposes, to levels between $level_{min}$ and $level_{max}$) during training to make predictions for all nodes, including those not seen during training~\cite{hamilton2020graph}. In our case, only some nodes are labeled with an MI score due to either systematic missing data or sampling policy (Section~\ref{sec:sampling}). Labeled nodes serve as input to the MSE loss computation, yet unlabeled nodes also participate in the message passing process.

A separate GNN model is generated for each subgroup. This strategy entails utilizing the entire multiplex graph structure for each model but restricting the loss computation to nodes belonging to a single subgroup at a time. The rationale behind this strategy is to allow each GNN model to focus solely on the specific characteristics and relationships within its corresponding subgroup, while still being influenced by nodes from other subgroups. By training separate models for each subgroup, we aim to enhance the model's ability to capture the nuanced patterns and dependencies unique to each subgroup, ultimately leading to more accurate predictions.  

\begin{algorithm}
	\caption{GNN Training Epoch}\label{alg:gnn_training_epoch}
	\begin{algorithmic}[1]
		\Require ${\mathbb{G}}$: A multiple lattice graph, with $\tilde{{\mathbb{V}}}=\bigcup_{i=1}^{|P|}\tilde{{\mathbb{V}_i}}$ nodes labeled with ground truth MI values $\tilde{\boldsymbol{MI}}^{GT}$
		\Require $\boldsymbol{\Theta}$: Model parameters (weight matrices)
		\Require $T$: Number of GNN layers
		
		\For{$i = 1$ to $|P|$}
		\For{$t = 1$ to $T$}
		\For{each node $v_i \in \mathbb{V}$}
		\State Compute $\boldsymbol{h}^t_{v_{i}}$ using Eq.~(\ref{eq:graphsage_neighbors}) and Eq.~(\ref{eq:graphsage_node_representation})\label{alg:representation}
		\EndFor
		\EndFor
		\For{each node $v_i \in \tilde{\mathbb{V}}_i$}
		\State $\hat{\boldsymbol{MI}}_{v_i} = \boldsymbol{W}_{i}^{out} \boldsymbol{h}^T_{v_{i}}$ \label{alg:predictMI}
		\EndFor

  \smallskip

      \State  $\mathcal{L}=\displaystyle \frac{1}{|\tilde{\mathbb{V}}_i|} \left \lVert \hat{\boldsymbol{MI}}_{v_{i}} - \tilde{\boldsymbol{MI}}^{GT}_{v_i} \right \rVert_2^2$
       \label{alg:lossComputation}
		
		\State Update $\boldsymbol{\Theta}_i$ based on the gradient of $\mathcal{L}$\label{alg:MRevision}
		\EndFor
	\end{algorithmic}
\end{algorithm}

A simplified sketch of a single training epoch is depicted in Algorithm~\ref{alg:gnn_training_epoch} for illustration purposes. The computation of revised node representations for {\bf all} nodes in the graph (assume $level_{min}=1$ and $level_{max}=|F^S|$) through message passing is performed in Line~\ref{alg:representation}. Then, we use these representations to predict the MI values for nodes in $\tilde{{\mathbb{V}_i}}$, for which we have computed the ground truth MI values (Line~\ref{alg:predictMI}). The average loss over these predictions is computed in Line~\ref{alg:lossComputation}, followed by an update of model parameters (Line~\ref{alg:MRevision}). 
It is noteworthy that one should not infer any computational complexity based on Algorithm~\ref{alg:gnn_training_epoch}, as all computations occur synchronously and the use of loops is given for clarity of presentation.

\subsection{Inferencing $TopK^m_i$}\label{subsection:topk}

\alg, trained on $\mathbb{G}$, serves as a robust tool for predicting MI scores for feature subsets. However. our analysis extends beyond, offering insights into the intricate relationships among feature subsets within. Leveraging the trained model, we flexibly estimate MI scores for uncomputed nodes within subgroups. Consequently, when presented with a subgroup $\mathbb{D}_i$, combination size $m$, and integer $K$, our objective is to retrieve $TopK^m_i$. This is done by retaining only nodes positioned at the $m$-th level at $\mathbb{G}_i$ and sorting them in a descending order with respect to their predicted MI scores.

\section{Empirical Evaluation}
\label{sec:experiments}
In this section, we test \algspace with real-world and synthetic datasets. We evaluate its efficacy by systematically varying core parameters, and compare it with multiple baselines of different kinds -- imputation-based, neural network-based, and Markov blanket after non-trivial adaptation.

Our experiments reveal four consistent observations that showcase the efficacy of our proposed solution. (1) \algspace outperforms the baselines consistently and exhibits robustness to significantly missing data compared to the baselines.
(2) Our proposed sampling strategy is both effective and efficient in capturing interdependencies among the feature sets and subgroups. (3) Computing MI of all possible feature subsets is a leading computational bottleneck that \alg\ tactically overcomes and scales under different varying parameters. (4) \algspace learns the problem semantics and in particular conserve to a large degree the upward closure property. 



Experimental setup (Section~\ref{subsection:experimental_setup}) is followed by an efficacy study of \alg, considering multiple baselines (Section~\ref{subsection:alg_preformance}). Section~\ref{subsection:sampling_analysis} delineates the effectiveness of our proposed sampling strategy, followed by an analysis of the upward closure property (Section~\ref{subsection:upward_closure}). 
Scalability analysis (Section~\ref{subsection:scalability}) concludes the empirical study. 

\subsection{Experimental Setup}
\label{subsection:experimental_setup}
We next discuss benchmark datasets (Section~\ref{subsection:datasets}), implementation details (Section~\ref{subsection:implementation_details}), baseline methods (Section~\ref{subsection:baselines}), and evaluation metrics (Section~\ref{subsection:evaluation_metrics}).
\subsubsection{Datasets}
\label{subsection:datasets}
We use three real-world and two synthetic datasets. The key metadata summary of the datasets is presented in Table~\ref{tab:datasets}.

\noindent \textbf{Real-world datasets.} We now describe the three publicly available datasets used in our experiments. 

\noindent {\tt Employee Attrition\cite{attrition}} predicts employee attrition (whether the employee stayed at work or left), based on (mostly) categorical features related to the employee's profession history (\emph{e.g.}, number of promotions, company tenure), personal circumstances (\emph{e.g.}, marital status, work-life balance), and job-related aspects (\emph{e.g.}, job satisfaction, job role, company size).
\ifTechReport
Three continuous features, namely \emph{Years at Company, Monthly Income} and \emph{Distance from Home} were discretized with the following range scheme:
$\textit{Years at Company}: \{\leq3, 4-6, 7-10, 11-20, 20<\}$, $\textit{Monthly Income (USD)}: \{\leq3,000, 3,001-5,000, 5,001-8,000, 8,001-10,000, 10,000<\}$, $\textit{Distance from Home (miles)}: \{\leq3, 4-6, 7-10, 11-20, 20<\}$.
\fi
We use age and gender as criteria for subgroup split, forming eight of them.
\ifTechReport
The age range splits are: $\{\leq25,~25-40 ~40-50,~50+\}$.
\fi

\noindent {\tt Mobile\cite{mobile}} classifies price ranges of different mobile phones using features like battery power, blue tooth capability, memory size, depth, weight and screen size. Subgroups are created by separating mobile phones with single and dual sims.

\noindent {\tt Loan}\cite{loan} classifies whether an individual will default on a loan payment using features like loan amount, interest rate, employment duration, home ownership, and payment plan. Subgroups are created based on loan grades, a measure of customer's financial credibility.

We discretize continuous variables through binning, a common practice in the MI literature~\cite{Battiti1994, peng2005feature}. 
For simplicity, we also discretize categorical features having more than 9 distinct values. For systematic missing data, we randomly select a subset of features within each subgroup, using missingness probability $p$. This method ensures our ability to evaluate the performance of \algspace against a valid ground truth. We ensure that every feature is present in at least one subgroup, and each subgroup contains at least one missing feature.


\ifTechReport
    \begin{table}
    	\centering
    	\caption{Metadata of the processed datasets}
    	\scalebox{1.5}{\begin{tabular}{|l|c|c|c|}
    			\hline
    			\textbf{Datasets} & \# records & $|F^S|$ & \# subgroups \\
                    \hline
                    Attrition & 59,598 & 19 & 8  \\
    			\hline
    			Mobile & 2,000 & 15 & 2  \\
    			\hline
    			Loan & 67,463 & 15 & 3  \\
    			\hline
    			$SD1$ & 50,000 & 15 & 4  \\
    			\hline
    			$SD2$ & 50,000 & 20 & 4  \\
    			\hline
    	\end{tabular}}
    	\label{tab:datasets}
    \end{table}
\fi

\noindent\textbf{Synthetic datasets.} we use synthetic datasets to control specific characteristics relevant to the studied problem. 
\ifTechReport
we use synthetic datasets to control specific characteristics of the data relevant to the studied problem. We predefine relevant features, through the use of digital logic to determine target variable values based on existing features. This logic, defined through Boolean algebra, introduces both linear and nonlinear relationships among feature subsets and target variables, and among different feature subsets. To accommodate complex scenarios beyond binary values, we perform these logical operations bitwise. For example, to allow four possible feature values we use two binary bits such that, for example, $01 \ XOR \ 11= 10$ and $01 \ AND \ 11= 01$.

Following~\cite{bolon2013review, kamalov2023synthetic}, we define four types of features, as follows. \emph{Relevant} features are those incorporated into the logical formula, directly influencing the definition of the target variable. \emph{Correlated} features are generated by randomly modifying the target variable's value at a predefined rate. \emph{Redundant} features are created as logical functions of relevant features. Lastly, \emph{irrelevant} features are randomly generated and hold no indicative significance with the target variable. The complexity level and distinctiveness between feature subsets in terms of MI are influenced by both the logical formula and the distribution of the different feature types. 

Subgroups are created by randomly partitioning the dataset into vertical fragments. Within each subgroup, feature values are randomly generated from a uniform distribution. Additionally, each subgroup uses a distinct random noise injection drawn from a normal distribution, involving value flips across its features. To simulate datasets with systematically missing data, we randomly select a set of features within each subgroup to make each of them empty with probability $p$. We employ three logical formulas to form two collections of hyperparameters ($SD1$ and $SD2$), differing in number of features, tuples in the dataset, random noise injection, {\em etc}. A comprehensive description of the hyperparameter settings and logical formulas is provided in our publicly available repository.\footnotemark[\value{footnote}] 
\fi

\begin{sloppypar}
\subsubsection{Implementation Details}
\label{subsection:implementation_details}
The experiments were executed on a Linux machine with NVIDIA A100 GPU. 
The GNN was implemented with PyTorch Geometric~\cite{fey2019fast}. Our entire code is publicly available in a GitHub repository.$^1$
\end{sloppypar}
We employ a heterogeneous variant of GraphSage~\cite{hamilton2017inductive} with two layers and a hidden representation dimensionality of $128$. Training lasts for $1000$ epochs, utilizing the \emph{Adam} optimizer~\cite{kingma2014adam} with a learning rate of $0.001$ and weight decay of $5e-4$. Model parameters are selected based on their performance on a validation set, which consists of 20\% of the training data randomly sampled. 

\subsubsection{Baselines}
\label{subsection:baselines} We use four baselines, as follows.


\noindent  (1) {\tt KNN}~\cite{Meesad2008CombinationOK}. An imputation based baseline, considering the entire feature space of every record and using hamming distance~\cite{han2012data} to compute the mode of the feature value of $k$-nearest neighbors ($k$ is an input parameter of the algorithm) whenever imputation is needed. Imputed values are used with the closed form of MI to quantify feature importance.

\noindent (2) {\tt Markov Blanket} A non-trivial adaptation of~\cite{impute2}, integrating missing data imputation inside Markov Blanket (MB) Learning. We first fill missing values using K-Nearest Neighbors (KNN) imputation. Then, using~\cite{impute2} we find the Markov Blanket (MB) of the target variable, which is a set of the most relevant features. The aforementioned two steps repeat until the MB no longer changes. Finally, the top features selected by MB are combined  to create sets of size $m$ by maximizing joint MI, from where the best $K$ results that have the  highest joint MI are retained.

\noindent (3) {\tt MLP} employs a fully-connected neural network using the binary representation of feature combinations (Section~\ref{subsection:single_lattice}) as inputs. The network uses two fully-connected layers with hidden dimension of 64 each. The model is trained against each subgroup separately, with the same parameter selection approach as \algspace (Section~\ref{subsection:implementation_details}).


\noindent (4) {\tt Arbitrary} performs a uniform random selection of samples from ${\mathcal F}$, to be compared against \algsamplenospace.



Evaluation is conducted with varying probability $p$ ($0.2$ and $0.5$) of features with systematic missing data, sampling rate $B \in \left\lbrace 0.25, 0.5, 0.75, 1.0\right\rbrace $ per subgroup, $m=3$ (number of designated features in a set), and $K \in \left\lbrace 5,10\right\rbrace$. 


\subsubsection{Evaluation Measures}
\label{subsection:evaluation_metrics}
\begin{sloppypar}
We use two rank-based measures to evaluate the effectiveness of the proposed solutions, namely nDCG@$K$ and precision@$K$. Sampling effectiveness is measured using  $\ell_1$ norm of total variation distance~\cite{chung1989measures}, as defined in Eq.~\ref{def:L1_total_var_dist}. We test on each subgroup separately and compute the average across all subgroups. The reported results are computed over three different seeds. The test set for each subgroup consists of the feature subsets that include at least one missing feature, alongside feature subsets that have not been sampled (Section~\ref{subsection:sampling_analysis}).
\end{sloppypar}

Given our special MI-based ranking system, we adjust the measures to our settings, as follows.

\begin{definition}[MI-based Normalized Discounted Cumulative Gain]
	\label{def:ndcg}
	Given a set of feature combinations of $m$-size:
\begin{equation}
	\begin{split}
		\textnormal{nDCG}@K &= \frac{\sum_{i=1}^K \frac{\mathds{1}_{\mathrm{rank}_{\mathrm{pred}}[i] \in \mathrm{TopK}^m}}{\log_2(i+1)}}{\sum_{i=1}^K \frac{1}{\log_2(i+1)}}
	\end{split}
\end{equation}

\noindent where $rank_{pred}$ is a descending order of predicted ranks and $\mathds{1}_{rank_{pred}[i] \ \in \ TopK^m}$ is an indicator, assigned $1$ if the $i$-th predicted rank is in the ground truth top-$k$ and $0$ otherwise.
\end{definition}

\begin{definition}[MI-based Precision]
	\label{def:precision}
	\begin{equation}
	\begin{split}
		\textnormal{precision}@K = \frac{rank_{pred}[:K] \cap TopK^m}{K}
	\end{split}
	\end{equation}    
\end{definition}

\begin{definition}[$\ell_1$ norm of total variation distance between two distributions]
	\label{def:L1_total_var_dist}
	Given two probability distribution function $S$ and $P$ defined on event space $\calE$, $\ell_1$ norm of total variation distance between $S$ and $P$ is defined as follows.
\begin{equation}
	\delta_{\ell_1}(S, P) = \sum_{e \in \calE} \lvert S(e) - P(e) \rvert
\end{equation}
\end{definition}




\subsection{\alg ~vs. Baselines}
\label{subsection:alg_preformance}
\noindent
Comparisons are performed using all datasets with no sampling strategy for \alg.
In this work, we consider a setting and problem definition that, to the best of our knowledge, are not present in the literature. Hence, to facilitate comparisons, we modified existing methods to serve as baselines.
For all datasets, $\alg$ outperforms the baselines according to both effectiveness measures, with multiple $K$ values.
\begin{sloppypar}
\smallskip \noindent {\bf Effectiveness.} Table~\ref{tab:res_table} shows the test results on the real-world and synthetic datasets, with a probability $p=0.2$ for a feature having systematic missing data. For the real-world datasets $\alg$ consistently outperforms all baselines in both nDCG@$K$ and precision@$K$.
As for the synthetic datasets,  which are rendered more difficult, $\alg$ showcases superior performance over the baseline, with few exceptions. 
{\tt MLP} exhibits the worst performance among the baselines. As the input for {\tt MLP} is initialized exactly as $\alg$, its inferior results indicate that without capturing the underlying dependencies between features, addressing the missing data challenge as a learning task is unfeasible. {\tt Markov Blanket} is better than {\tt MLP} but worse than {\tt KNN}. From these experimental analyses a clear message prevails - our proposed solution \alg\ outperforms all the baselines. Among the baseline solutions, {\tt KNN} turns out to be the most effective one. The rest of the comparison in the experiment section therefore is conducted between {\tt KNN} and \alg.
\end{sloppypar}

\begin{table}
	\caption{\small Effectiveness comparison of \alg\ and the baselines, demonstrating that \alg\ exhibits higher nDCG and precision compared to the baselines for almost all $K$.} 
	\centering
	\scalebox{1.05}{\begin{tabular}{c|l|cccc}
			\hline
			\multirow{2}{*}{\textbf{Dataset}} & \multirow{2}{*}{\textbf{Metrics}} & \multicolumn{4}{c}{\textbf{Algorithm}} \\
			& & \alg & {\tt KNN} & {\tt Markov} & {\tt MLP}  \\
			\hline

                \multirow{4}{*}{Attrition} 
			& nDCG@5 & \underline{0.39} & 0.29 & 0.34 & 0.04 \\
			& nDCG@10 & \underline{0.53} & 0.34 & 0.39 & 0.07 \\
			& precision@5 & \underline{0.36} & 0.25 & 0.33 & 0.04 \\
			& precision@10 & \underline{0.49} & 0.29 & 0.37 & 0.08 \\
			\hline
            
			\multirow{4}{*}{Mobile} 
			& nDCG@5 & \underline{0.64} & 0.53 & 0.05 & 0.06 \\
			& nDCG@10 & \underline{0.74} & 0.60 & 0.07 & 0.19 \\
			& precision@5 & \underline{0.64} & 0.47 & 0.04 & 0.07  \\
			& precision@10 & \underline{0.67} & 0.53 & 0.07 & 0.20  \\
			\hline
			\multirow{4}{*}{Loan} 
			& nDCG@5 & \underline{0.56} & 0.45 & 0.23 & 0.02 \\
			& nDCG@10 & \underline{0.64} & 0.54 & 0.34 & 0.02  \\
			& precision@5 & \underline{0.51} & 0.42 & 0.18 & 0.02  \\
			& precision@10 & \underline{0.59} & 0.52 & 0.30 & 0.02 \\
			\hline
			\multirow{4}{*}{SD1} 
			& nDCG@5 & \underline{0.58} & 0.45 & 0.47 & 0.02  \\
			& nDCG@10 & \underline{0.68} & 0.54 & 0.50 & 0.05 \\
			& precision@5 & \underline{0.52} & 0.39 & 0.40 & 0.02  \\
			& precision@10 & \underline{0.61} & 0.47 & 0.42 & 0.06  \\
			\hline
			\multirow{4}{*}{SD2} 
			& nDCG@5 & 0.52 & \underline{0.58} & 0.54  & 0.03 \\
			& nDCG@10 & \underline{0.64} & 0.61 & 0.59 & 0.06  \\
			& precision@5 & 0.47 & \underline{0.56} & 0.47 & 0.03 \\
			& precision@10 & \underline{0.57} & 0.54 & 0.55 & 0.06 \\
			\hline
	\end{tabular}}
	\label{tab:res_table}
\end{table}

\ifTechReport
\begin{figure}[!htbp]
	\includegraphics[scale=0.5]{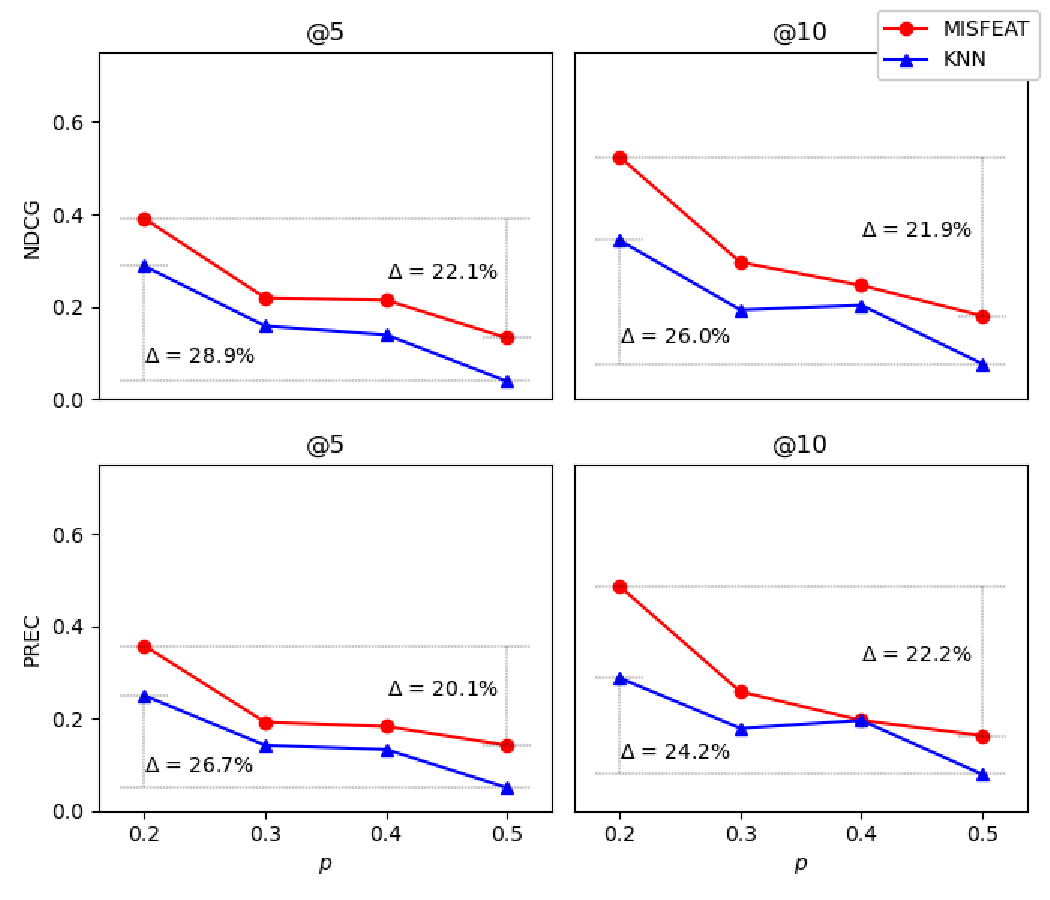}
	\vspace{-11pt}
	\caption{\small({\tt Attrition dataset}) nDCG and Precision with increasing $p$. \alg\ is consistently more effective with smaller $\Delta$ values. }
	\label{fig:attrition_Delta}
\end{figure}
\fi
\ifTechReport
\begin{figure}[!htbp]
	\includegraphics[scale=0.5]{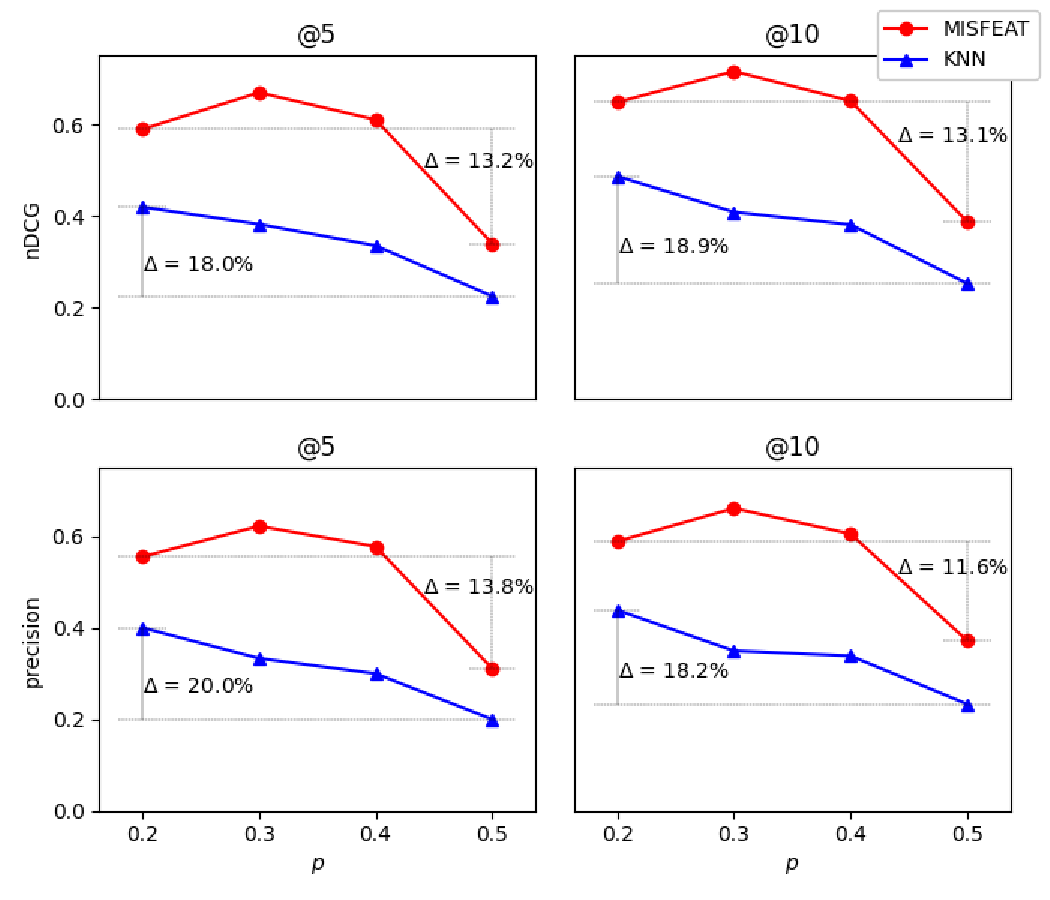}
	\vspace{-11pt}
	\caption{\small({\tt Mobile dataset}) nDCG and Precision with increasing $p$. \alg\ is consistently more effective with smaller $\Delta$ values. }
	\label{fig:mobile_Delta}
\end{figure}
\begin{figure}[!htbp]
	\includegraphics[scale=0.5]{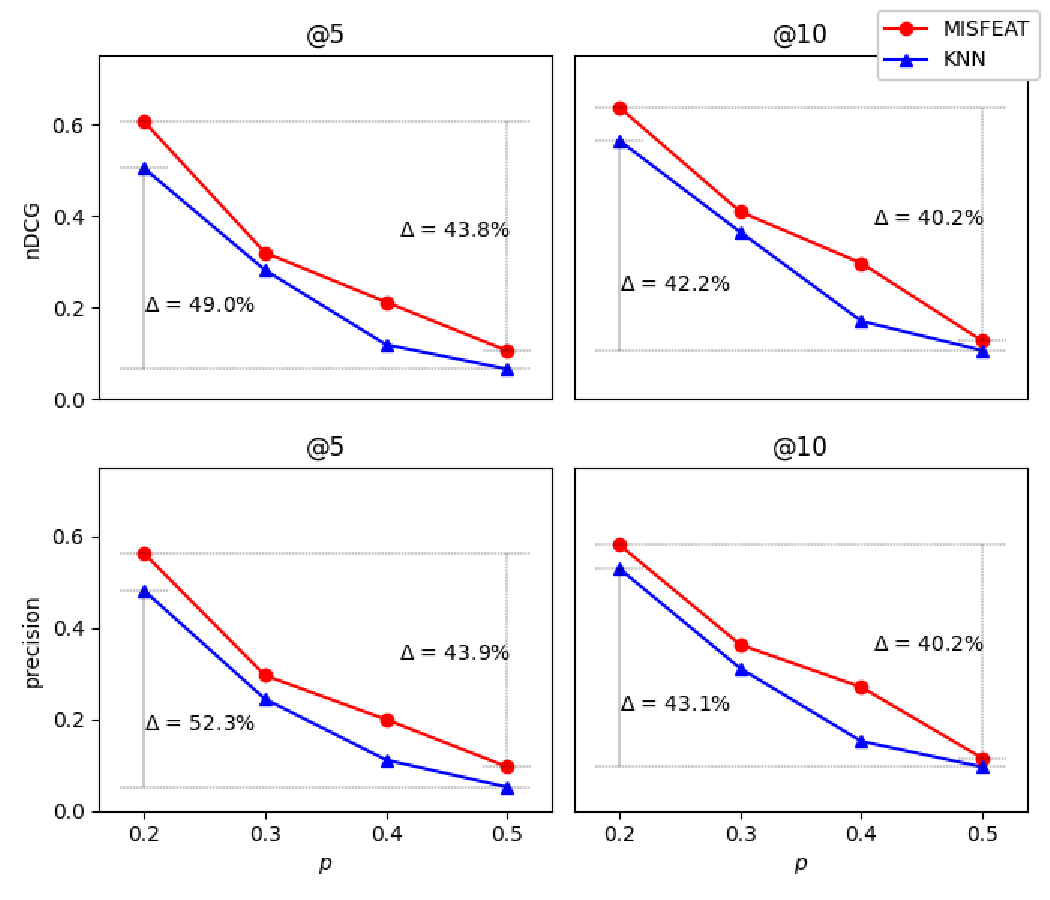}
	\vspace{-11pt}
	\caption{\small ({\tt Loan dataset}) nDCG and Precision with increasing $p$. For both algorithms, $\Delta$ is higher (although \alg\ performs better) due to low correlation of features across subgroups.}
	\label{fig:loan_Delta}
\end{figure}
\fi

\begin{sloppypar}
\smallskip \noindent {\bf Evaluating Robustness.} 
In this experiment, we compare \alg\ and {\tt KNN} with increasing $p$ (likelihood of a feature to have systematically missing data). As expected, both $\alg$ and the baselines including {\tt KNN} perform worse as $p$ increases. We display effectiveness on the real-world datasets and the average relative drop $\Delta$ for 10\% increase in $p$. For nDCG the average drop is calculated as: $\Delta = 1/3 \times \{\frac{nDCG@p=0.2 - nDCG@p=0.3}{nDCG@p=0.2}+\frac{nDCG@p=0.3 - nDCG@p=0.4}{nDCG@p=0.3}+\frac{nDCG@p=0.4 - nDCG@p=0.5}{nDCG@p=0.4}\}$. The average drop in precision is calculated analogously.
As can be seen in Figure~\ref{fig:attrition_Delta}, \alg\ is more robust to high values of $p$, reflected by its lower $\Delta$ in both measures. This highlights the effectiveness of \alg\ in mitigating the impact of systematically missing data, utilizing its message passing mechanism across subgroups to capture feature set dependencies and reuce the adverse effect of missingness.
\ifTechReport
Figures~\ref{fig:mobile_Delta} and \ref{fig:loan_Delta}, demonstrate similar trends, showing that \alg\ is more robust to high values of $p$. 
The drop is higher for both algorithms in the {\tt Loan} dataset, although \alg\ always outperforms {\tt KNN}. The difference can be attributed to lower correlation across subgroups for this dataset. 
\fi
\end{sloppypar}

\subsection{Sampling Analysis}
\label{subsection:sampling_analysis}
The goal of \algsample is to produce a uniform random sample from the distribution of all possible feature sets (population). We evaluate the effectiveness of \algsample vs {\tt Arbitrary} on this regard by measuring the $\ell_1$ norm of total variation distance of MI between \algsample and population and that of {\tt Arbitrary}  and population.

\begin{figure}[htbp]
    \includegraphics[scale=0.448]{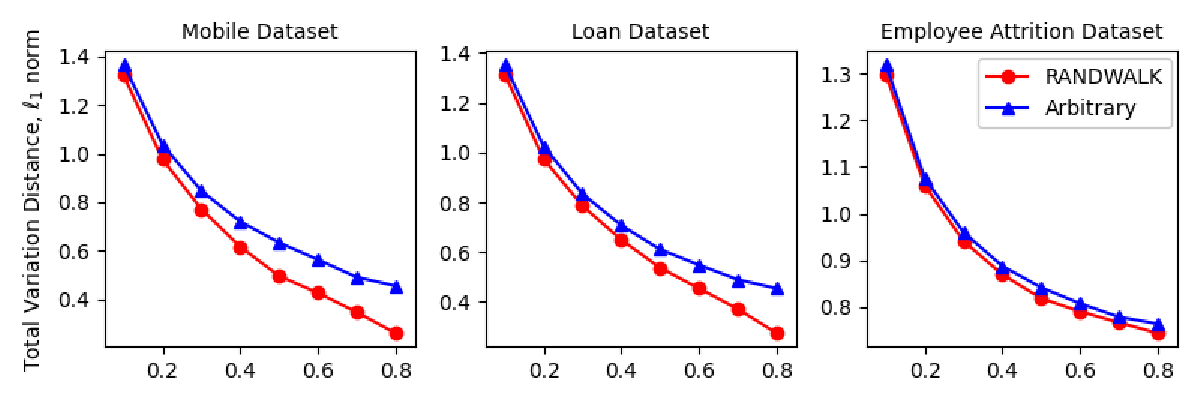}
    \caption{\small $\ell_1$ norm of total variation distance of MI between $\algsample$ and population distribution vs. that of {\tt Arbitrary} and population distribution varying sampling budget $B$. \algsample\ consistently shows lower variation distance.} 
    \label{fig:mobile_R_vs_A}
\end{figure}
Figure~\ref{fig:mobile_R_vs_A} illustrates a comparative analysis of $\algsample$ and {\tt Arbitrary} in terms of $\ell_1$ norm of total variation distance on {\tt Mobile}, {\tt Loan}, and {\tt Employee Attrition} datasets. Per Definition~\ref{def:L1_total_var_dist}, $S$ represents the distribution of MI for sampled nodes, while $P$ represents the distribution of MI for the entire lattice. Evidently, $\algsample$ outperforms {\tt Arbitrary} in all datasets with lower $\delta_{\ell_1}$ across different budgets $B$, hence reliably represents the characteristics of the distribution of MI across all nodes in the lattice.

\subsection{The Upward Closure Property with \algspace}

\label{subsection:upward_closure}
The upward closure property states that the MI of a smaller feature subset is never larger than that of any of its supersets (Section~\ref{subsection:FeatureSelection}). In the context of our work, this property is manifested by the lattice structure, where inter-level edges (Section~\ref{subsection:single_lattice}) indicate the inclusion relationship between subsets of subsequent levels in the lattice. When all MI scores are available, every node has a higher MI score than its neighbors from the level below reflecting the upward closure property. \algspace utilizes the graph structure and learns its topology via the message passing mechanism.

\ifTechReport
\begin{figure}[htpb]
\center
	\includegraphics[scale=0.51]{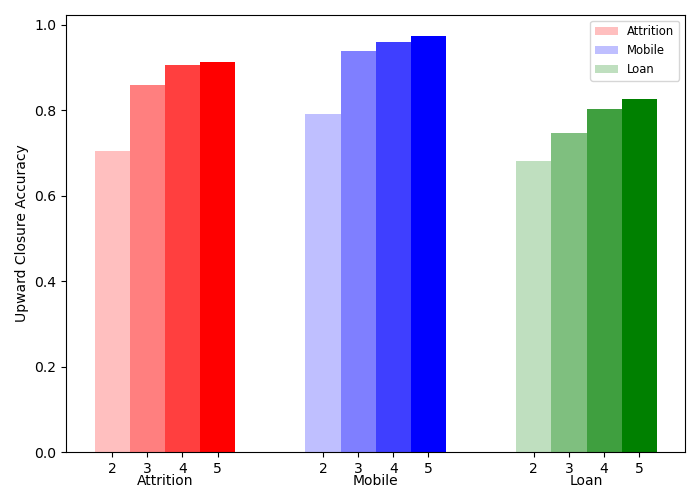}
	\caption{Upward closure accuracy scores for the real-world data sets over increasing number of levels}
	\label{fig:upward_closure_accuracy_analysis}
\end{figure}
\fi

We compute the upward closure prediction accuracy as a ratio of the number of times a node was predicted with a higher MI than its neighbor from a lower level to the total number of edges connecting between subsequent levels in the lattice.
The empirical results for \algspace are shown in Figure~\ref{fig:upward_closure_accuracy_analysis}, where levels are represented by distinct bar, and a group of four adjacent bars refers to the same dataset.
As can be seen, \algspace gradually improves its performance and feature subsets become more inclusive as the level number increases. This highlights the model's ability to leverage the lattice's structural properties and the relationships between feature subsets.

\subsection{Scalability of \alg}
\label{subsection:scalability}
To identify computational bottlenecks and study the efficiency of $\alg$, we compare it against {\tt KNN}, an imputation-based baseline. 
To ensure a fair comparison between the two approaches, we compare the time required by {\tt KNN} to impute missing data and compute MI over ${\mathcal F}^-$ (feature subsets with missing features) against the training and inference time of \alg. Note that \algspace eliminates the needs for MI computation of ${\calF}^-$ during inference by leveraging the trained GNN model.
\ifTechReport
\begin{figure}[htbp]
	\centering
	\includegraphics[scale=0.55]{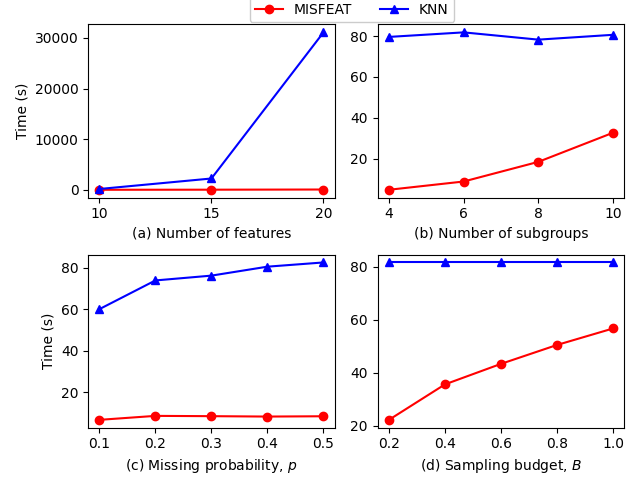}
	\caption{Execution time: \algspace  vs. {\tt KNN} by varying different parameters, demonstrating that \algspace scales well.}
	\label{fig:scalability} 
\end{figure}
\fi

We evaluate scalability of the algorithms using a synthetic dataset of $50,000$ records. In each experiment, we vary the value of a single parameter while keeping the others constant ($60\%$ sampling strategy, missing probability $p=0.2$, $|F^S|=10$, and $4$ subgroups). Figure~\ref{fig:scalability} presents the execution time comparison of $\alg$ and {\tt KNN} by varying each of the four parameters. 
Figure~\ref{fig:scalability}a focuses on varying the number of features. Since {\tt KNN} inference requires computing MI after imputation, it is susceptible to the exponential growth of feature subsets when the number of features increases. $\alg$, on the other hand, benefits from fast training and inference using GNN, resulting in a significant speedup compared to {\tt KNN}. 
In Figure~\ref{fig:scalability}b we vary the number of subgroups. Since {\tt KNN} does not differentiate between subgroups, its execution time remains stable as their number increases. \alg, on the other hand, demonstrates its scalability and effectively handles a large number of subgroups. 
Figure~\ref{fig:scalability}c compares varying values of missing probability $p$. As $p$ increases, the size of $\calF^-$ grows, requiring {\tt KNN} to spend more time on imputating MI values for $\calF^-$. In contrast, \algspace benefit from its inferencing mechanism and increasing $p$ does not add to its overhead.
Finally, Figure~\ref{fig:scalability}d  shows the results across increasing sampling budget, $B$. With a lower budget, \algspace requires less time for training whereas the execution time of {\tt KNN} is unaffected by the budget, as expected. For \alg, reducing $B$ from $1.0$ to $0.2$ results in 3$\times$ speedup on execution time.

\ifTechReport
\begin{figure}[htbp]
	\centering
	\includegraphics[scale=0.57]{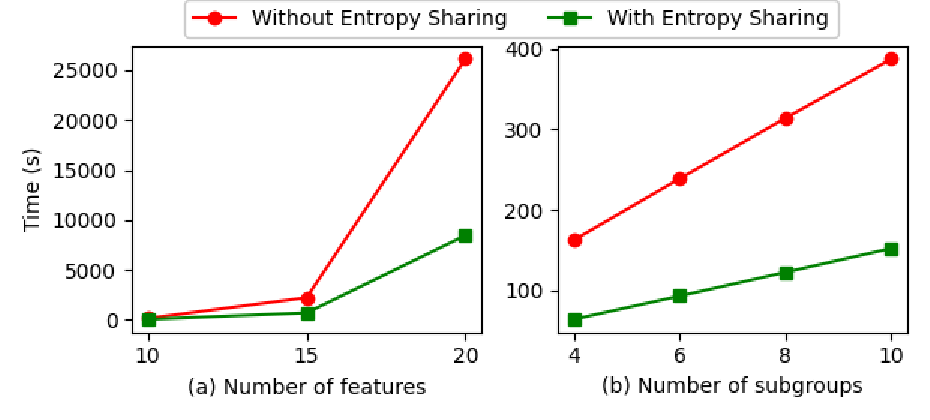}
	\caption{MI pre-computation with and without entropy sharing.}
	\label{fig:pre-compute} 
\end{figure}
\fi

To speed up the process of pre-computing MI, we use entropy sharing when computing MI values over the lattice, using $\Theta(2^{|\calF|})$ storage space (see~\ref{sec:MI-sharing}). Figure~\ref{fig:pre-compute} illustrate the effectiveness of this approach by varying the number of features (Figure~\ref{fig:pre-compute}a) and the number of subgroups (Figure~\ref{fig:pre-compute}b). Entropy sharing reduces significantly the MI computation time  with increasing number of features and subgroups. 

We can conclude that exhaustive enumeration of MI is computationally infeasible and $\alg$ scales robustly across varying numbers of subgroups, missing probabilities, and sampling budgets.



\section{Related Work}
\label{sec:related_work}
There are three common methods for feature selection. Filtering uses statistical measures to rank and select features~\cite{Koller1996TowardOF,Hancer2018}.  Wrapper methods utilize predictive performance of a specific learning algorithm~\cite{Mafarja2018WhaleOA,JOHN1994121}. Finally, embedded methods, such as LASSO, integrate feature selection directly into the model training process to enhance model generalization~\cite{Robert1996,Hui2005,Yuan2006}. 
No related work, to the best of our knowledge, studies feature selection for systematic missingness considering different subgroups. We next analyze the related work on the use of MI for feature selection, feature selection with missing data, and the use of deep learning for the task.

\noindent\textbf{Feature selection \& MI.}   
Existing feature selection algorithms are typically categorizes into four main groups~\cite{li2017feature}: similarity-based, information-theoretical-based, sparse-learning-based, and statistical-based methods. MI is a model agnostic filtering-based information-theoretic approach that is widely used in feature selection~\cite{battiti1994using,brown2012,Vergara2014,salam2019human}. It quantifies the dependency between variables, thereby assisting in selecting the most informative features for predicting the target variable~\cite{battiti1994using,brown2012,Vergara2014}. 
\cite{Hanchuan2005}~\cite{Hanchuan2005} introduce a comprehensive framework that utilizes MI to achieve optimal feature selection by maximizing dependency on the target variable, enhancing relevance of the features, and minimizing redundancy among them. {\em We borrow inspiration from these prior works and consider MI for selecting important feature sets.} 
\begin{sloppypar}
\noindent\textbf{Feature selection \& missing data.}
Meesad and Hengpraprohm~\cite{Meesad2008CombinationOK} 
propose the use of k-NN based missing value imputation to improve feature selection. 

Yu et al.~\cite{impute2} introduce a novel framework for causal feature selection with missing data, integrating multiple imputation and Markov blanket learning to enhance both data imputation and causal discovery in Bayesian networks. The proposed graphical model integrates multiple imputations ({\em not typically suitable for systematic missing data}) and Markov blanket learning to enhance both data imputation and causal discovery of features. This work is different from ours in two main aspects. First, {\em we use prediction rather than imputation}. Second, {\em no support for subgroups or systematic computation of top-$K$ feature subsets, each with $m$ features, is provided}. We non-trivially adapt this solution to serve as a baseline. 

Xue et al.~\cite{Xue2022} have developed a multi-objective approach to feature selection that effectively handles missing data in classification tasks, optimizing for both feature relevance and robustness of the selection process. Zhu et al.~\cite{ZHU2018488} have developed a method for multi-label feature selection that effectively addresses the challenge of missing labels, ensuring the robustness and accuracy of feature selection in complex multi-label environments. Prior work has studied a method for feature selection with missing data using MI estimators that directly estimate MI from incomplete datasets, bypassing the need for imputation and preserving the integrity of the original data distribution~\cite{mi1}. Qian and Shu~\cite{qian2015mutual} studied MI-based feature selection method for incomplete data that combines tolerance information granules and a forward greedy strategy for efficiency purposes. {\em Other than the k-NN based approach (implemented as a baseline with inferior performance), none of these techniques handle systematic missing data to produce top-$K$ feature sets with a predefined number of features.}
\end{sloppypar}   

\noindent\textbf{Feature selection \& Deep learning.} Gradient-based methods, such as DeepLIFT, deduce feature significance through changes in gradients observed during back-propagation across network layers~\cite{Shrikumar2017}. Lu et al.~\cite{Lu2018} present DeepPINK, a methodology that employs filter technique to enhance the reproducibility of feature selection in deep neural networks, focusing on reliable identification of significant features across different datasets. Furthermore, methods like Integrated Gradients offer alternative techniques by attributing the prediction of a neural network to its inputs, thereby providing a deeper understanding of feature importance, which complements these methods in complex model architectures~\cite{Sundararajan2017}. 
{\em These works are not filtering-based and cannot produce top-$K$ feature sets based on MI.} Belghazi et al~\cite{belghazi2018mutual} introduce MINE, estimating MI using a neural network-based discriminator. While their work also tackles computational challenges using sampling {\em it is not aimed to deal with missing data.}
\section{Conclusions}
\label{sec:conclusion}
 We introduced \alg, a GNN-based framework 
 for feature selection considering MI with systematic missing data for datasets with distinct subgroups. The proposed model is generalizable. It can handle both systematic and random missing data and can be extended to handle multiple model agnostic feature selection measures. The proposed model is based on lattice organization of feature subsets, benefiting from the MI upward closure property, which it learns well, and targeted sampling of MI computation of a limited number of feature subsets, using multiple efficiency opportunities in training the model. Through a thorough empirical analysis, of both real-world and synthetic datasets, we demonstrate the effectiveness of \alg and its different components. In particular, we demonstrate that the efficiency opportunities attain significant speedup. We also show the designed solution accurately predicts missing MI values, even under severe sampling budget limitations. Also, the top-$K$ feature subsets correlate well with the true ranking of feature subsets, offering a useful decision-making mechanism for applications in domains such as medical informatics, where data gathering may be costly or otherwise restricted by regulatory bodies.



\bibliographystyle{IEEEtran}
\bibliography{main.bib}

\end{document}